\documentclass[11pt]{article}

\usepackage[letterpaper,top=1in, bottom=1in, left=1in, right=1in]{geometry}
\usepackage[utf8]{inputenc}
\usepackage[hidelinks]{hyperref} %

\usepackage{amsmath}
\usepackage{amsthm}
\usepackage{amsfonts}
\usepackage{mathrsfs}
\usepackage{algorithm}
\usepackage{algorithmic}
\usepackage{mathtools}
\usepackage{bm}
\usepackage[hidelinks]{hyperref}
\usepackage[dvipsnames]{xcolor}
\usepackage{tikz}

\usepackage[suppress]{color-edits}
\addauthor{tdw}{OrangeRed}
\addauthor{rj}{blue}
\addauthor{an}{teal}

\usepackage{authblk}
\usepackage[round]{natbib}

\theoremstyle{plain}
\newtheorem{theorem}{Theorem}[section]
\newtheorem{proposition}[theorem]{Proposition}
\newtheorem{lemma}[theorem]{Lemma}
\newtheorem{corollary}[theorem]{Corollary}
\theoremstyle{definition}

\newtheorem{assumption}[theorem]{Assumption}
\theoremstyle{remark}
\newtheorem{remark}[theorem]{Remark}

\renewcommand{\Pr}{\mathbb{P}}
\newcommand{\E}{\mathbb{E}}
\newcommand{\cS}{\mathscr{S}}
\newcommand{\cA}{\mathscr{A}}
\newcommand{\cO}{\mathscr{O}}
\newcommand{\cT}{\mathscr{T}}

\newcommand{\bigOtilde}{\Tilde{\mathcal{O}}}
\newcommand{\bigO}{\mathcal{O}}

\newcommand{\RegBound}{\bigOtilde\left(\alpha^{-2}H^{2.5} S^2 A^{0.5}O^{1.5}\sqrt{(1+S A/O)K}  \right)}

\DeclareMathOperator*{\argmax}{arg\,max}

\allowdisplaybreaks
\mathtoolsset{showonlyrefs}

\begin{document}
\title{\bf Posterior Sampling-based Online Learning for Episodic POMDPs}
\author[1]{Dengwang Tang\thanks{Work done while at University of Southern California}}
\author[2]{Dongze Ye}
\author[2]{Rahul Jain\thanks{Also affiliated with Google DeepMind. Work done at University of Southern California}}
\author[2]{Ashutosh Nayyar}
\author[2]{Pierluigi Nuzzo}
\affil[1]{eBay, Inc.}
\affil[2]{University of Southern California}
\date{}

\maketitle
\begin{abstract}
  Learning in POMDPs is known to be significantly harder than in MDPs. In this paper, we consider the online learning problem for episodic POMDPs with unknown transition and observation models. We propose a Posterior Sampling-based reinforcement learning algorithm for POMDPs (PS4POMDPs), which is much simpler and more implementable compared to state-of-the-art optimism-based online learning algorithms for POMDPs. We show that the Bayesian regret of the proposed algorithm scales as the square root of the number of episodes and is polynomial in the other parameters. In a general setting, the regret scales exponentially in the horizon length $H$, and we show that this is inevitable by providing a lower bound. However, when the POMDP is undercomplete and weakly revealing (a common assumption in the recent literature), we establish a polynomial Bayesian regret bound. We finally propose a posterior sampling algorithm for multi-agent POMDPs, and show it too has sublinear regret.
\end{abstract}

\section{Introduction}

Markov Decision Process (MDP) models have proved to be remarkably effective at representing sequential decision-making problems. When the model is known, dynamic programming provides an elegant framework for  planning algorithms for MDPs. Unfortunately, in practical problems the MDP model is often unknown. This has necessitated the development of reinforcement learning (RL) methods to deal with such settings. Such RL methods, when combined with deep learning architectures, have then yielded a number of remarkably effective deep RL algorithms, such as DQN \citep{dqn,dqn2} and PPO \citep{ppo}, for very large state and action spaces.  Often, when the model is unknown, there is a need for systematic exploration to facilitate efficient learning. Online learning methods exactly provide such systematic exploration. The antecedent of many online reinforcement learning algorithms is actually in algorithms for bandit learning \citep{ucb1,russo2018tutorial}. These algorithms can broadly be classified into two categories: \textit{optimism-based} \citep{lai1985asymptotically,auer2002nonstochastic,auer2008near}, and \textit{posterior sampling (PS)-based} \citep{thompson1933likelihood,chapelle2011empirical}).  While there is an elegant theory for both, often PS algorithms are found to have superior empirical performance. On the whole, the developments over the last decade have led to a fairly well-established theory of RL for MDPs, and effective algorithms for both offline and online settings.

However, in many practical problems, the state is not fully observable, and partially observable Markov decision processes (POMDPs) provide more accurate representations. Even when the transition and observation models are known, the planning problem for POMDPs is computationally challenging, as it requires conversion to an equivalent MDP problem over the belief state space, which is a continuum \citep{kumar2015stochastic}. Nevertheless, the belief state is still a sufficient statistic of the history, and many approximate planning algorithms \citep{shani2013survey,pomcp} are based on this observation. When the model parameters are unknown, the learning for POMDPs problem becomes much harder than for MDPs since the belief state itself can no longer be computed. Thus, unlike their planning counterparts, POMDP learning problems cannot be simply reformulated as belief-state-MDP learning problems. This represents significant challenges in an online learning setting when the learning agent must perform systematic exploration, and accounts for the lack of progress in designing \textit{effective} and \textit{implementable} online learning algorithms for POMDPs.

In this paper, we consider episodic reinforcement learning problems on finite-horizon POMDPs with finite state, action, and observation spaces. The exact models of the transition and observation kernels are unknown to the learning agent. We propose a  Posterior Sampling-based reinforcement learning algorithm for POMDPs (PS4POMDPs), which is an adaptation of the PS method used for bandit learning \citep{agrawal2012analysis,russo2016information} and MDP learning problems \citep{osband2013more, russo2013eluder, ouyang2017learning}.  PS-based algorithms show superior empirical performance over optimism-based approaches in a number of bandit and MDP learning settings \citep{agrawal2012analysis,osband2017posterior,ouyang2017learning, ouyang2019posterior, jafarnia2021online,jafarnia2021learning,jahromi2022online,kalagarla2023safe}. 
\rjedit{If this is any guide, we can expect a similar observation for POMDP learning. However, there are currently no practical online algorithms for POMDPs, whether based on posterior sampling or not.}
Compared to PS in MDPs, our algorithm updates the posterior on both the transition kernel and the observation kernel. We analyze the Bayesian regret of the PS4POMDPs algorithm in two settings, namely, (1) the general case, where no assumption on the POMDP is imposed, and (2) the undercomplete $\alpha$-revealing POMDPs \citep{jin2020sample,liu2022partially}, which quantify the requirement that the observations must be informative to a certain degree. We show that in general POMDP learning problems, the regret is $\mathrm{poly}(S, A, O, H)\cdot\bigOtilde(\sqrt{(OA)^H K})$, where $K$ is the number of episodes and $H$ is the horizon length. We show that the exponential dependence on $H$ is necessary by proving an $\Omega(\sqrt{A^{H-1} K})$ lower bound of the regret. Under the assumption that the POMDP is undercomplete and $\alpha$-revealing, we establish an $\RegBound$ upper bound on the Bayesian regret.  

The \textit{main contributions} in this paper are the following: (i) We introduce a posterior sampling-based online reinforcement learning algorithm for POMDPs (PS4POMDPs) that is \textit{simple and implementable}, and yet has a provable sublinear regret bound that scales as $\bigOtilde(\sqrt{K})$, where $K$ is the number of learning episodes. The regret matches the known lower bound \citep{chen2023lower} while being polynomial in the other factors under the weakly-revealing assumption. A PS-based learning algorithm for POMDPs was first proposed in \cite{jahromi2022online}, but is computationally impractical.  Our algorithm  assumes access to a POMDP solver and computationally tractable posterior updates, thus building on the extensive literature on computationally tractable approximate POMDP solvers \citep{shani2013survey} and  posterior update methods  \citep{russo2018tutorial}. In Section \ref{sec:empirical}, we provide empirical results on one such approximate implementation of the PS4POMDPs algorithm and show that it can achieve sublinear regret. (ii) We establish an $\RegBound$ upper bound on the Bayesian regret, which is the best known for a posterior sampling algorithm for POMDPs (see \cite{jahromi2022online}) and  close to similar regret bounds for optimism-based algorithms \citep{liu2022partially,chen2023lower}. (iii) \rjedit{As an additional technical contribution, we introduce a novel, tighter index change lemma (see Appendix A), which may prove useful in other contexts as well. Furthermore, our analysis depends on the introduction of an appropriate confidence set, which is specific to posterior-sampling algorithms and different from those used in optimism-based algorithms \citep{liu2022partially,chen2023lower}. Thus, Lemmas \ref{lem:step0main} and \ref{lem:step1main} are new. }(iv) We finally introduce the first online learning algorithm for multi-agent POMDPs, and characterize its performance in terms of regret bounds. Although this extension is non-trivial, due to space constraints, it is \rjedit{reported in Appendix \ref{app:mapomdps}}.

\section{Related Literature on POMDPs}\label{sec:litt}
There have been many works focusing on reinforcement learning of episodic POMDPs. In \cite{ross2007bayes,poupart2008model} and \cite{ross2011bayesian}, the authors considered Bayesian models for POMDP learning and proposed several algorithms but without sample complexity or regret guarantees.  
Closely related to the POMDP learning problem is the parameter estimation problem of Hidden Markov Models (HMMs), a special case of POMDPs without actions. A number of POMDP learning algorithms \citep{guo2016pac,azizzadenesheli2016reinforcement,xiong2022sublinear} have been developed based on the spectral method (for HMMs \citep{hsu2012spectral,anandkumar2014tensor})  but assume that both the state and observation kernels have full rank. In contrast, we only need the observation kernels to have full rank. Furthermore, \cite{guo2016pac} only provides sample complexity analysis while \cite{xiong2022sublinear} find a $\bigOtilde(T^{2/3})$\footnote{In this section, $\bigOtilde$ or $\bigO$ hides instance-dependent constants and dependence on $S, A, O, H$ or other related parameters.} instance-dependent regret bound where $T$ is the learning horizon, equivalent to $KH$. \cite{azizzadenesheli2016reinforcement} have a regret bound that scales linearly with the inverse of the smallest singular value of the transition kernel (and hence is unsatisfactory). \cite{lee2023learning} considered POMDP learning problems with hindsight observability, where the states are revealed to the learning agent at the end of each episode. We make no such assumption.

\cite{jahromi2022online} applied the PS approach
to infinite horizon POMDPs. They established an $\bigO(\log T)$ instance-dependent regret bound in the finite parameter case under certain assumptions, where $T$ is the number of learning instances. In the general case, they established an $\bigOtilde(T^{2/3})$ regret bound assuming the existence of a consistent transition kernel estimator with a specific convergence rate. However, finding such an estimator is often one of the key difficulties in POMDP learning problems. \rjedit{Moreover, it requires keeping track of a certain hard-to-compute “pseudo count” for every state-action pair, and makes other assumptions that are complicated and unverifiable}.

\rjedit{Our work is closest to \citep{jin2020sample,zhan2022pac,liu2022partially,liu2023optimistic,zhong2022gec,chen2022partially,chen2023lower}, where online learning for episodic POMDPs (or the more general predictive state representation models) is considered. These papers all assumed $\alpha$-weakly revealing POMDPs, or its generalization. \cite{jin2020sample,zhan2022pac,liu2022partially}, and \cite{liu2023optimistic} proposed optimism-based algorithms while \cite{chen2022partially,chen2023lower,zhong2022gec} proposed and analyzed optimistic posterior sampling algorithms. However, all such algorithms are computationally impractical or non-implementable either in the optimism step, or the posterior update, or both. They either require solving joint optimization problems over the parameter and the policy space in each episode, or the computation of optimal values under \emph{all} possible parameters in order to form the optimistic posterior distribution. Furthermore, \cite{jin2020sample, zhan2022pac, chen2022partially} only provided sample complexity  results (not regret bounds) for algorithms that are not practical or even implementable. Our results improve upon the regret bound in \cite{liu2022partially} (after normalizing to our setting) by $\Omega(H^2\sqrt{SA})$. The regret results in \cite{liu2023optimistic,zhong2022gec} are more difficult to compare due to differences in the assumptions made. Regret bounds for three algorithms presented in \cite{chen2023lower} are comparable (or sightly better) but the algorithms are  practically non-implementable (i.e., there is no way to implement them even without worrying about computationally tractability). The focus of our work is to introduce an implementable, close-to-practical algorithm with theoretical guarantees rather than improving the theoretical understanding of limits for learning POMDPs.}

\section{Preliminaries}\label{sec:prelim}

\textbf{Notations.}
For a positive integer $n$, $[n] := \{1,2,\cdots, n\}$. For two integers $t_1\leq t_2$, we use $t_1:t_2$ to indicate the collection of indices $\{t_1, t_1+1, \cdots, t_2\}$. For example, $a_{1:4}$ stands for the vector $(a_1, a_2, a_3, a_4)$. For a finite set $\varOmega$, $\Delta(\varOmega)$ is the set of probability distributions on $\varOmega$. $\bm{1}_{\mathcal{E}}$ is the indicator function of the event $\mathcal{E}$. $\mathbf{e}_{i}$ represents a unit vector where the $i$-th entry is 1 and all other entries are $0$. The dimension of $\mathbf{e}_{i}$ is inferred from the context in which it is used. $\mathbf{I}$ represents an identity matrix whose dimension is also inferred from the context. For finite sets $\varOmega_1, \varOmega_2$, if a function $g$ has the form $\varOmega_1 \mapsto \Delta(\varOmega_2)$, we write $g(\omega_2|\omega_1) := [g(\omega_1)](\omega_2)$ as if $g$ represents a conditional probability measure. Similarly, if $g: \varOmega_1 \mapsto \Omega_2$ then we write $g(\omega_2|\omega_1) := \bm{1}_{\{ g(\omega_1) = \omega_2\}}$. $\|\mu_1 - \mu_2\|_{\mathrm{TV}}$ represents the total variation distance between probability distributions $\mu_1$ and $\mu_2$. For $p> 0$, $\|\cdot\|_p$ is the standard $\ell_p$-norm for vectors and $\ell_p$ induced norm for matrices. $\E$ and $\Pr$ stands for expectation and probability, respectively. All logarithms in this paper are natural logarithms. Finally, we use $\bigO, \Omega$ to represent the standard big-O and big-Omega notation, which hides the absolute constants (i.e., constants independent of any problem parameters). We also use $\bigOtilde, \Tilde{\Omega}$ to hide absolute constants along with logarithmic factors of problem parameters.

We now consider reinforcement learning problems where a learning agent repeatedly interacts with the same environment in multiple episodes. The environment can be described as a finite horizon POMDP with parameters only partially known to the learning agent.

\textbf{The Environment Model.} A finite POMDP is characterized by a tuple $(\cS, \cA, \cO, H, b_1, T, Z, r)$, where $\cS$ is a finite set of states with $|\cS| = S$; $\cA$ is a finite set of actions with $|\cA| = A$; $\cO$ is a finite set of observations with $|\cO| = O$; $H$ is the horizon length; $b_1 \in\Delta(\cS)$ is the distribution of the initial state; $T = (T_h)_{h=1}^{H-1}, T_h: \cS\times\cA\mapsto \Delta(\cS)$ are the transition probabilities; $Z = (Z_h)_{h=1}^{H}, Z_h: \cS\mapsto \Delta(\cO)$ are the observation probabilities; $r = (r_h)_{h=1}^{H}, r_h: \cO\times\cA \mapsto [0, 1]$ are the instantaneous reward functions. For each POMDP characterized by the above tuple, we also define the following matrices: $\mathbb{T}_{h, a} = (T_h(s'|s, a))_{s'\in\cS, s\in\cS} $ is the $S\times S$ probability transition matrix (where the rows represent the next state) under action $a\in\cA$ at time $h$; $\mathbb{Z}_h = (Z_h(o|s))_{o\in\cO, s\in\cS}$ is the $O\times S$ observation probability matrix at time $h$.

A (deterministic) policy $\pi = (\pi_h)_{h=1}^H$ is a collection of mappings $\pi_h: (\cO\times \cA)^{h-1}\times \cO \mapsto \cA$, where $\pi_h$ is the mapping an agent uses to choose an action at time $h\in [H]$ based on action and observation history in the current episode.  
Let $\varPi$ denote the space of all deterministic policies.

A trajectory $\tau = (o_h, a_h)_{h=1}^H$ is the complete action and observation history in a single episode. Let $\cT = (\cO \times \cA)^H$ denote the set of trajectories. Under a policy $\pi\in\varPi$, the probability of a trajectory $\tau = (o_h, a_h)_{h=1}^H$ is given by $\Pr^{\pi}(\tau) = \pi(\tau) \Pr^{-}(\tau)$, where
\begin{align}
    \pi(\tau)&:= \prod_{h=1}^H \pi_h(a_h|\tau_{h-1}, o_h)\label{eq:policypart} \\
    \Pr^{-}(\tau) &:= \sum_{s_{1:H}\in\cS^H } \left[b_1(s_1) Z_H(o_H|s_H)\prod_{h=1}^{H-1} Z_h(o_h|s_h) T_h(s_{h+1}|s_h, a_h)\right]\label{eq:envpart}
\end{align}
where $\tau_h \in (\cO\times\cA)^{h}$ is the partial trajectory made up with the first $h$ observations and actions in $\tau\in\cT$. The above representation is particularly helpful for our analysis since it separates the ``policy part'' from the ``environment part.'' The environment part can also be written in terms of matrix multiplications as follows
\begin{align}
    \Pr^-(\tau) &= \mathbf{e}_{o_H}^T \mathbb{Z}_H \mathbb{T}_{H-1, a_{H-1}} \mathrm{diag}(\mathbb{Z}_{H-1}(o_{H-1}, \cdot))\times\cdots\times \mathbb{T}_{1, a_{1}} \mathrm{diag}(\mathbb{Z}_1(o_{1}, \cdot)) b_1 
    \label{eq:likelihood_computation}
\end{align}
where $\mathrm{diag}(\mathbf{w})$ is a diagonal matrix whose main diagonal is given by the vector $\mathbf{w}$. \noeqref{eq:policypart}\noeqref{eq:envpart}

The expected total reward in one episode under policy $\pi\in\varPi$, or the value of a policy $\pi$, is given by
\begin{equation}
    V^{\pi} := \sum_{o_{1:H}, a_{1:H} }\left( \Pr^{\pi}(o_{1:H}, a_{1:H}) \sum_{h=1}^H r_h(o_h, a_h) \right).
\end{equation}

The maximum total expected reward in one episode over all policies $\pi\in\varPi$, or the value of the POMDP, is defined as $V^* = \max_{\pi\in\varPi} V^{\pi}$.

\textbf{Learning Agent's Prior Knowledge.} We assume that $\cS, \cA, \cO, H, r$ are known to the learning agent. The quantities $b_1, T$ and $Z$ are (in general) unknown to the agent. We assume that $b_1, T, Z$ are parameterized by a parameter $\theta\in \varTheta$, and the learning agent knows the parameterization (i.e., the agent knows the set $\varTheta$, and what $(b_1^{\theta}, T^{\theta}, Z^{\theta})$ is for each given $\theta\in\varTheta$). The learning agent's prior knowledge of the true environment $\theta^*$ is modeled by a distribution\footnote{
More formally, we assume that $\Theta$ is a Borel subset of some $\mathbb{R}^d$, the prior distribution $\nu^1$ is a Borel measure, and the parameterization mapping $\theta \mapsto (b_1^{\theta}, T^{\theta}, Z^{\theta})$ is Borel measurable.} $\nu^1\in \Delta(\varTheta)$. In the rest of the paper, we view $\theta^*$ as a primitive random variable with distribution $\nu^1$. We will also add a subscript $\theta$ to the quantities defined above (e.g., $\Pr_{\theta}^{\pi}(\tau), \Pr_{\theta}^{-}(\tau), V_{\theta}^{\pi}, V_{\theta}^*$) to signify that they are associated with the POMDP $(\cS, \cA, \cO, H, b_1^{\theta}, T^{\theta}, Z^{\theta}, r)$. 

\textbf{Learning Agent's Interaction with the Environment.} At the beginning of each episode, the learning agent chooses a potentially randomized policy in $\varPi$ based on past trajectories and policies. More specifically,
for $k\in\mathbb{N}$, let $\mathcal{D}_k := (\tau^j, \pi^j )_{j=1}^{k-1} = (o_{1:H}^j, a_{1:H}^j, \pi^j)_{j=1}^{k-1}$ denote the data which the learning agent possesses at the beginning of the $k$-th episode, composed of trajectories and policies in the first $k-1$ episodes. At the beginning of the $k$-th episode, the learning agent chooses a random policy $\pi^k \sim \phi_k(\mathcal{D}_k)$ via a mapping $\phi_k: (\cT\times\varPi)^{k-1}\mapsto \Delta(\varPi)$ and applies this policy throughout the $k$-th episode. We refer to $\phi = (\phi_k)_{k\in\mathbb{N}}$ as a \emph{learning algorithm}. (We use this term to distinguish it from the term \emph{policy}, which we use exclusively for local mappings in one episode.)

\textbf{Objectives.} Given the prior belief $\nu^1\in\Delta(\varTheta)$, the Bayesian regret of a learning algorithm $\phi$ over $K$ episodes is defined as
\begin{align}
\mathrm{BReg}(\phi, K) &:= \E^{\phi}\left[\sum_{k=1}^K (V_{\theta^*}^* - V_{\theta^*}^{\pi^k})\right] = \int_{\theta\in\varTheta}\E_{\theta}^{\phi}\left[\sum_{k=1}^K (V_{\theta}^* - V_{\theta}^{\pi^k})\right]\mathrm{d}\nu^1(\theta)
\end{align}
which measures the difference between the maximum possible reward when one knows $\theta^*$ and the actual reward realized by the learning algorithm $\phi$. The goal of the learning agent is to choose a learning algorithm with small Bayesian regret.

\textbf{Proposed Algorithm.} We consider the Posterior Sampling-based reinforcement learning algorithm for POMDPs (PS4POMDPs). In this algorithm, the learning agent keeps a posterior belief on the true parameter $\theta^*$ through Bayesian updates at the end of each episode, i.e., at the end of episode $k$, after utilizing the policy $\pi^k\in\varPi$ and observing the trajectory $\tau^k$, the agent computes $\nu^{k+1}\in\Delta(\varTheta)$ via
\begin{equation}\label{eq:nubayes}
    \dfrac{\mathrm{d}\nu^{k+1}}{\mathrm{d}\nu^k}(\theta) := \dfrac{\Pr_{\theta}^{\Tilde{\pi}^k}(\tau^k)}{\int_{\theta'\in \varTheta} \Pr_{\theta'}^{\Tilde{\pi}^k}(\tau^k) \mathrm{d}\nu_k(\theta') }.
\end{equation}
We also assume that the agent has access to 
a POMDP planner, which returns an optimal policy for a given POMDP. Such an assumption has been widely used in POMDP learning literature \citep{jin2020sample,liu2022partially,liu2023optimistic,chen2022partially,jahromi2022online,xiong2022sublinear,zhong2022gec}.
A more detailed description of the algorithm is given by Algorithm \ref{algo:psrl}.

\begin{algorithm}[!ht]
   \caption{PS4POMDPs}
   \label{algo:psrl}
\begin{algorithmic}
   \STATE \textbf{Input:} Prior $\nu^1\in\Delta(\varTheta)$; Number of episodes $K$
	 \FOR{$k = 1$ to $K$} 
	    \STATE Sample $\Tilde{\theta}^k \sim \nu^k$ 
	    \STATE Use the POMDP planner to obtain a policy $\Tilde{\pi}^k\in \arg\max_{\pi\in\varPi}(V_{\Tilde{\theta}^k}^{\pi})$ 
        \STATE Apply $\Tilde{\pi}^k$ in the $k$-th episode
        \STATE 
	    Collect the trajectory $\tau^k$ and compute new posterior $\nu^{k+1} \in\Delta(\varTheta)$ using \eqref{eq:nubayes}
	\ENDFOR
\end{algorithmic}
\end{algorithm}

\textbf{Assumptions on the Environment.}
In this paper, we analyze the PS4POMDPs algorithm in two different settings. The first setting is the general case, i.e., no assumptions are imposed on the underlying POMDP. The second setting is the setting of undercomplete $\alpha$-weakly revealing POMDPs, which was first introduced by \cite{jin2020sample} in the context of POMDPs. \tdwedit{For completeness, we introduce the full definition in Assumption \ref{assump:obsfullrank}.}

\begin{assumption}\label{assump:obsfullrank}
    \citep{jin2020sample} The observations are \textit{undercomplete}, i.e., $O\geq S$, and the POMDP is \textit{$\alpha$-weakly revealing}, i.e., for all $\theta\in\varTheta$ and all $h\in [H]$, the smallest singular value of the $O\times S$ 
    observation probability matrix $\mathbb{Z}_h^{\theta}$ satisfies $\sigma_{\min}(\mathbb{Z}_h^{\theta}) \geq \alpha$.
\end{assumption}

Intuitively, Assumption \ref{assump:obsfullrank} states that the observations must give a reasonable amount of information about the underlying state \citep{jin2020sample}, \tdwedit{ruling out the difficult classes of POMDPs where observations under certain different states are virtually indistinguishable}. In \cite{golowich2023planning}, the authors show that an $\varepsilon$-optimal policy of a weakly revealing POMDP can be found with a quasipolynomial time planner.

\begin{remark}
\rjedit{Algorithm \ref{algo:psrl} is conceptually simple and implementable. It requires use of a POMDP solver and there are many excellent ones available \citep{shani2013survey}. It also requires a computationally tractable procedure for approximate posterior update such as those available in  \citep[Chapter 5]{russo2018tutorial}, though it turns out that Bayesian inference algorithms such as Mixed Hamiltonian Monte Carlo (MHMC) suffice for the experimental work in Section \ref{sec:empirical}.}
\end{remark}

\section{Main Results}\label{sec:mainresults}
In this section, we formally state our main results. The first two results concern the general case, where no assumptions are imposed on the POMDP. We defer the proof outline of the results to the next section. 

\begin{theorem}\label{thm:breg0}
    For general POMDP learning problems, the Bayesian regret under the PS4POMDPs algorithm satisfies
    \begin{align*}
       \mathrm{BReg}(\phi^{\mathrm{PS4POMDPs}}, K) &\leq \bigOtilde\left(H^2\sqrt{(S^2A + SO) (OA)^H K}\right) 
    \end{align*}
\end{theorem}
The proof can be found in Appendix \ref{app:breg0}.

\begin{remark}
    Utilizing the fact that the trajectory probability can be separated into a product of the ``policy part'' and the ``environment part,'' the episodic POMDP learning problem can be seen as a special case of linear bandit problem with dimension $d = (OA)^H$. Applying the standard result on the posterior sampling algorithm for linear bandits \citep{russo2016information} we obtain an $\bigOtilde(H\sqrt{O^{2H+1} A^H K})$ regret bound, where the additional $O^H$ comes from the fact that $|\varPi| = \Omega(A^{O^H})$. The same regret bound can also be obtained if the LinUCB algorithm \citep{linucb} is applied instead \citep{lattimore2020bandit}. Theorem \ref{thm:breg0} presents an improvement over this naive bound by a factor of $\Tilde{\Omega}\left(\sqrt{\frac{O^{H+1}}{H^2(S^2A + SO)}} \right)$.
\end{remark}

The next result shows that, in the general case, the exponential dependence on $H$ in the regret bound is unavoidable under any learning algorithm.

\begin{proposition}\label{thm:breglower}
    For any $A, H\geq 2$ and any $K \geq A^{H-1}$, there exists a POMDP learning problem with $S=O=2$ such that the Bayesian regret satisfies
    \begin{equation}
        \mathrm{BReg}(\phi, K) \geq \frac{1}{20}\sqrt{A^{H-1} K}.
    \end{equation}
    under any learning algorithm $\phi$.
\end{proposition}

The proof is available in Appendix \ref{app:A}.

Next, we state our regret bound in the second setting, where the POMDP is assumed to be undercomplete and $\alpha$-weakly revealing.

\begin{theorem}\label{thm:breg1}
    Under Assumption \ref{assump:obsfullrank}, the Bayesian regret under PS4POMDPs algorithm satisfies
    \begin{align*}
        \mathrm{BReg}(\phi^{\mathrm{PS4POMDPs}}, K) &\leq \RegBound
    \end{align*}
\end{theorem}

The proof can be found in Appendix \ref{app:breg1}.

\begin{remark}\label{remark:restrict}
    Due to the complexity of solving POMDPs even with known parameters and the exponential growth of action and observation history, sometimes one may consider a restricted policy set $\Tilde{\varPi}\subset \varPi$ (e.g., finite memory policies). Both Theorem \ref{thm:breg0} and Theorem \ref{thm:breg1} will continue to hold in this setting if the optimization oracle used in Algorithm \ref{algo:psrl} returns a best restricted policy, and the regret is defined with respect to the best restricted policy.
\end{remark}

Finally, we state our regret upper bound when an approximate POMDP planner instead of an optimal planner is used in Algorithm \ref{algo:psrl}.

\begin{corollary}\label{cor:approx}
    Suppose that in episode $k$ of Algorithm \ref{algo:psrl} an $\epsilon_k$-optimal planner is used instead of the optimal planner. Then the Bayesian regret under Algorithm \ref{algo:psrl} satisfies
    \begin{align*}
        \mathrm{BReg}(\phi^{\mathrm{PS4POMDPs}}, K) &\leq \bigOtilde\left(H^2\sqrt{(S^2A + SO) (OA)^H K}\right) + \sum_{k=1}^K \epsilon_k,
    \end{align*}
    and if Assumption \ref{assump:obsfullrank} holds,
    \begin{align*}
        \mathrm{BReg}(\phi^{\mathrm{PS4POMDPs}}, K) &\leq \RegBound + \sum_{k=1}^K \epsilon_k.
    \end{align*}
\end{corollary}

\begin{proof}
    Using the fact that $\Tilde{\theta}^k$ has the same distribution as $\theta^*$ and $\Tilde{\pi}^k$ is $\epsilon_k$-optimal w.r.t. $\Tilde{\theta}^k$, the expected episode $k$ regret satisfies 
    \begin{equation}
        \E[V_{\theta^*}^*] - \E[V_{\theta^*}^{\Tilde{\pi}^k}] = \E[V_{\Tilde{\theta}^k}^*] - \E[V_{\theta^*}^{\Tilde{\pi}^k}] \leq \E[V_{\Tilde{\theta}^k}^{\Tilde{\pi}^k}] + \epsilon_k - \E[V_{\theta^*}^{\Tilde{\pi}^k}].
    \end{equation}

    Therefore we have
    \begin{equation}
        \mathrm{BReg}(\phi^{\mathrm{PS4POMDPs}}, K) \leq \left(\sum_{k=1}^K \E[V_{\Tilde{\theta}^k}^{\Tilde{\pi}^k} - V_{\theta^*}^{\Tilde{\pi}^k}]\right) + \sum_{k=1}^K \epsilon_k
    \end{equation}
    where the first term can be bounded with the same proof of either Theorem \ref{thm:breg0} or Theorem \ref{thm:breg1}.
\end{proof}

\section{Proof Outline}\label{sec:proofoutline}
In this section we lay out the proof outline for the results stated in Section \ref{sec:mainresults}. The proof details are available in the Appendices.

\subsection{Regret Lower Bound}
Proposition \ref{thm:breglower} is motivated by the pathological POMDP example used by \cite{krishnamurthy2016pac} and \cite{jin2020sample}. In this POMDP, the first $H-1$ actions act as ``rotating dials'' to a  ``combination lock'': One must enter a specific sequence in order to ``unlock'' at time $H$. The first $H-1$ observations are completely uninformative, so that the learning agent has no way to learn if the entered sequence is correct or not until the very last step. 
Such a POMDP resembles a multi-armed bandit with $A^{H-1}$ arms. \cite{krishnamurthy2016pac} and \cite{jin2020sample} established sample complexity lower bounds on these POMDPs. However, these results cannot be directly translated into lower bounds on the cumulative regret. In the proof of Proposition \ref{thm:breglower}, we apply the standard 
change of measure technique \citep{auer2002nonstochastic} on modified POMDP examples from \cite{krishnamurthy2016pac} and \cite{jin2020sample}.

\subsection{Regret Upper Bound}
Our proofs of both theorems follow similar steps as in the result from  \cite{liu2022partially}. However, due to complications arising from the randomization in the learning algorithm, the proof steps of \cite{liu2022partially} cannot be directly adapted. At the center of our proofs is a confidence set $\Bar{\varTheta}(\mathcal{D}_k)$, which is a finite set of parameters whose log-likelihood given the data $\mathcal{D}_k$ is close to the maximum log-likelihood. The set is constructed such that with high probability, the true parameter $\theta^*$ is close to some parameter in $\Bar{\varTheta}(\mathcal{D}_k)$. The specific definition of $\Bar{\varTheta}(\mathcal{D}_k)$ is provided in Appendix \ref{app:regub}.
The proof of both theorems will be based on Lemma \ref{lem:step0main}, where we relate the regret bound to the quality of the confidence set, and Lemma \ref{lem:step1main}, where we establish certain guarantees on the confidence set. 

\begin{lemma}\label{lem:step0main}
    The Bayesian regret under the PS4POMDPs algorithm can be bounded by
    \begin{align}
        \mathrm{BReg}(\phi^{\mathrm{PS4POMDPs}}, K) &\leq 2H + H \E\left[\sum_{k=1}^K   \max_{\Bar{\theta} \in \Bar{\varTheta}(\mathcal{D}_k) }\|\Pr_{\Bar{\theta}}^{\Tilde{\pi}^k} - \Pr_{\theta^*}^{\Tilde{\pi}^k}\|_{\mathrm{TV}} \right]\label{eq:bregtvmain}
    \end{align}
    where $\Pr_{\theta}^\pi$ is understood as a probability measure on $\cT$, the set of action-observation trajactories.
\end{lemma}

The proof of Lemma \ref{lem:step0main} mostly follows from standard regret decomposition techniques for PS4POMDPs algorithms (see e.g., \cite{osband2013more,russo2014learning}) and properties of the log-likelihood function. The proof can be found in Appendix \ref{sec:step0}.

On the other hand, through a martingale defined with the log-likelihood function, we prove Lemma \ref{lem:step1main}, which provides a guarantee on the quality of the confidence set.

\begin{lemma}\label{lem:step1main}
    Under any learning algorithm $\phi$, with probability at least $1 - \frac{1}{K}$, 
    \begin{align*}
        \max_{k\in [K]}\max_{\Bar{\theta}\in\Bar{\varTheta}(\mathcal{D}_k) } \sum_{j=1}^{k}  \|\Pr_{\Bar{\theta}}^{\pi^j} - \Pr_{\theta^*}^{\pi^j}\|_{\mathrm{TV}}^2 \leq  \bigOtilde(HS^2A + HSO)
    \end{align*}
\end{lemma}

The proof of Lemma \ref{lem:step1main} is in Appendix \ref{sec:step1}.

Given the above lemmas, to obtain an upper bound on the Bayesian regret, the only remaining task is to use an upper bound of
\begin{align}
    \sum_{j=1}^{k} \|\Pr_{\check{\theta}^k}^{\Tilde{\pi}^j} - \Pr_{\theta^*}^{\Tilde{\pi}^j}\|_{\mathrm{TV}}^2\label{eq:MLEmain}
\end{align}
to derive an upper bound of
\begin{align}
    \sum_{j=1}^{K}\|\Pr_{\check{\theta}^j}^{\Tilde{\pi}^j} - \Pr_{\theta^*}^{\Tilde{\pi}^j}\|_{\mathrm{TV}}\label{eq:REGmain}
\end{align}
where $\check{\theta}^k$ is a parameter in the confidence set $\Bar{\varTheta}(\mathcal{D}_k)$.
The key difference in the above two expressions lies in the episode index of the ``environment estimator'' $\check{\theta}_k$: The former measures the difference of the \emph{latest} estimated environment with the true environment under historical policies, while the latter measures the cumulative difference derived from the estimated environments \emph{at the time}. To complete this task, we will make use of the following ``index change'' lemma, which is a corollary of the elliptical potential lemma used in the linear bandit literature (see, e.g., \cite{lattimore2020bandit}).

\begin{lemma}[Index Change]\label{lem:exchangeindexlitemain}
    Let $x_k, w_k \in\mathbb{R}^d$ be such that $\|w_k\|_2\leq G_w,~\|x_k\|_2\leq G_x$ for all $k\in [K]$. Suppose that $\sum_{j=1}^{k} (x_j^T w_k)^2\leq \beta$ for all $k\in [K]$. Then for any $\lambda > 0$,
    \begin{align}
        \sum_{k=1}^K |x_k^T w_k| \leq  \sqrt{(\lambda + \beta)dK\log\left(1 + \dfrac{G_w^2G_x^2K}{d\lambda}\right)}.
    \end{align}
\end{lemma}
The proof of the lemma can be found in Appendix \ref{app:changeofindex}. 

The proof of Theorem \ref{thm:breg0} follows from a direct application of Lemma \ref{lem:exchangeindexlitemain} on certain suitably defined $(OA)^H$-dimensional vectors. Note that the dimension of the vectors is the reason for the exponential dependence on $H$.

To obtain a regret polynomial in $H$ under Assumption 1, we follow the same three-step procedure by \cite{liu2022partially}, where we use an auxiliary quantity called the \emph{projected operator distances}. Under Assumption 1, those distances can be used to both upper and lower bound total variation distances between trajectory distributions. In the first two steps, we establish relationships between \eqref{eq:MLEmain} and \eqref{eq:REGmain} and projected operator distances. In the final step, we apply (a more sophisticated version of) the ``index change'' lemma on expressions involving projected operator distances. 
Since each projected operator distance can be represented by inner products of $S$-dimensional rather than $(OA)^H$-dimensional vectors, it allows us to obtain a better constant for the regret bound. 

We conclude this section by noting a technical, but important distinction of our technique from that of \cite{liu2022partially} in the proof of Theorem \ref{thm:breg1}. Instead of using Proposition 22 in \cite{liu2022partially}, which is based on $\ell_1$-eluder dimension theory, we develop a new result (Proposition \ref{lem:exchangeindexdeluxe} in Appendix \ref{app:changeofindex}) based on the standard elliptical potential lemma. Our technical tool is arguably easier to prove and comes with a tighter guarantee, which ultimately enables us to improve the upper bound by \cite{liu2022partially} by a factor of $\Tilde{\Omega}(H^2\sqrt{SA})$. We note that the tighter regret bound is not an advantage of the PS4POMDPs algorithm, as our new technique can also be applied in the original analysis of \cite{liu2022partially} for their optimism-based algorithm. We also note that a similar technique was developed in \cite{chen2022partially,chen2023lower} independently.

\section{Empirical Results}\label{sec:empirical}

In this section, we show that PS4POMDPs can be practically implemented with a standard POMDP solver, SARSOP \citep{kurniawati2009sarsop}, and Bayesian inference algorithms such as the No-U-Turn Sampler (NUTS) \citep{hoffman2014NUTS} and Mixed Hamiltonian Monte Carlo (M-HMC) \citep{zhou2020mixed}. Our results show that, despite approximation errors, PS4POMDPs has excellent empirical performance in terms of expected return and  the \textit{frequentist} cumulative expected regret, i.e., $\mathrm{Reg}_{\theta}(K) = KV_{\theta}^\star - \sum_{k=1}^K V_{\theta}^{\Tilde{\pi}^k}$. \rjedit{Unfortunately, none of the algorithms in  \cite{jafarnia2021online, liu2022partially, chen2022partially}, etc. can be practically implemented for comparison even on a toy problem.}

\paragraph{Tiger (Finite Horizon).} 
We consider a finite horizon version of a widely used POMDP problem called the \texttt{Tiger} problem with $H=10, S = 5, A=3, O=5$ (details in Appendix \ref{app:tiger}). Our problem has an unknown observation kernel parameterized by a single parameter $\theta\in [0, 0.5]$ . Our algorithm is evaluated on three \texttt{Tiger} instances with true parameter values $\theta^\star = 0.2, 0.3, 0.4$, respectively. In each case, we use a truncated normal prior on $\theta^\star$ with mean $0.25$ and variance $0.25$ over the interval $[0.1, 0.5]$. Figure \ref{fig:expresult_tiger} shows  the regret of \texttt{PS4POMDPs} divided by $K$ (the number of episodes), and by $\sqrt{K}$ to illustrate the order of the growth rate. We observe that the regret for our algorithm is clearly growing sub-linearly. As seen in the second plot of  Figure \ref{fig:expresult_tiger}, when scaled by $\sqrt{K}$, the  cumulative regret seems to converge to a constant thus indicating that it is scaling as $\sqrt{K}$.

\begin{figure}[!ht]
    \centering
    \includegraphics[width=0.7\linewidth]{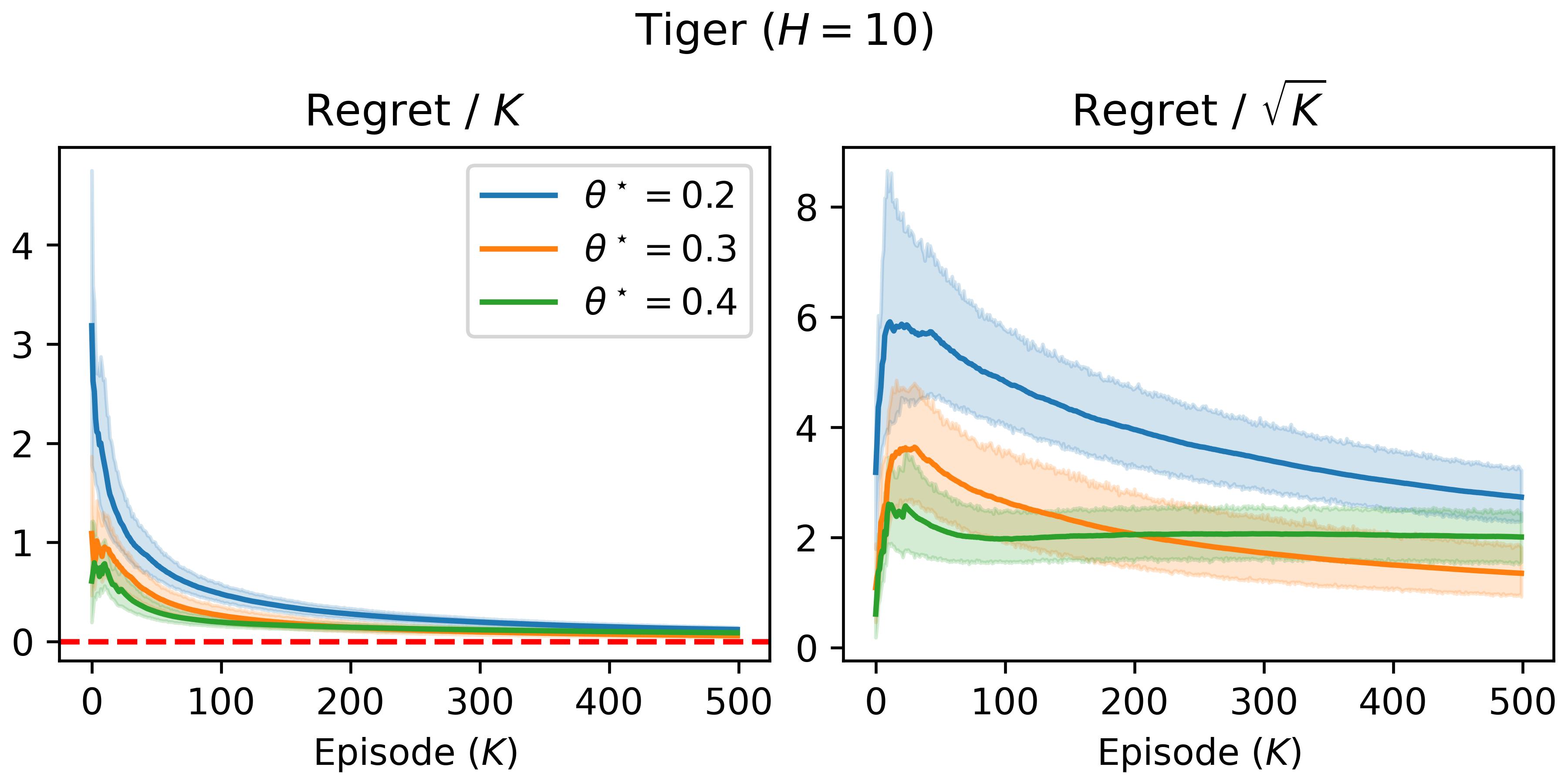}
    \caption{Cumulative expected regret on \texttt{Tiger} with 95\% confidence interval over 20 runs. (a) Left: Regret divided by the number of episodes $K$. (b) Right: Divided by  $\sqrt{K}$.}
    \label{fig:expresult_tiger}
\end{figure}

\paragraph{RiverSwim (POMDP).} We introduce sensor errors to the \texttt{RiverSwim} problem \citep{strehl2008analysis, osband2013more} and consider the corresponding POMDP with $H=40, S=6, A=2, O=6$ (details in Appendix \ref{app:riverswim}). The transition kernel of our \texttt{RiverSwim} is identical to the MDP version considered by \cite{osband2013more}. We assume no prior knowledge on the environment and use a Dirichlet prior  with pseudo-counts 1, which forms a uniform distribution over all possible transition and observation kernels. We compare the performance of PS4POMDPs against the near-optimal policy found by SARSOP for the true model. Figure \ref{fig:expresult_river_swim} shows  the expected regret of our algorithm (against the near-optimal policy) divided by $K$ and $\sqrt{K}$. The expected return for each policy $\tilde{\pi}^k$ used in PS4POMDPs  is computed via Monte Carlo estimation and then averaged over 20 sample paths. 

\begin{figure}[!ht]
    \centering
    \includegraphics[width=0.7\linewidth]{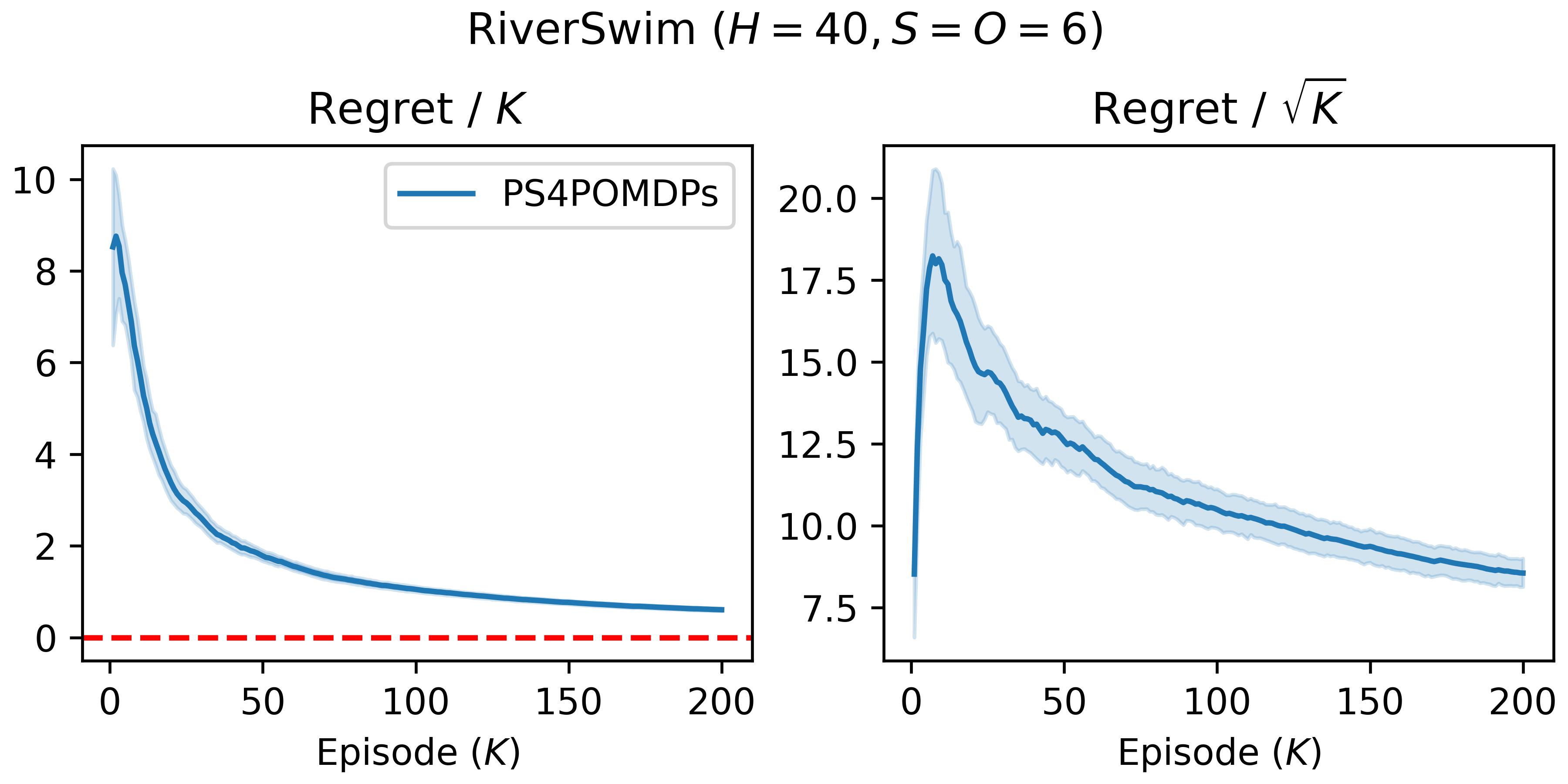}
    \caption{Cumulative expected regret on \texttt{RiverSwim} with 95\% confidence interval over 20 runs. 
    (a) Left: Regret divided by the number of episodes $K$. (b) Right: Divided by  $\sqrt{K}$.
    }
    \label{fig:expresult_river_swim}
\end{figure}

\paragraph{Additional Experiments.} We also evaluate the performance of \texttt{PS4POMDP} in (a) RiverSwim POMDP with different levels of prior knowledge on the structure of transition and observation kernels, and (b) in random POMDPs (generated from the prior) in Appendix \ref{app:riverswim}.

\paragraph{Computational Details.}
Our algorithm uses the NUTS and M-HMC algorithms implemented in NumPyro \citep{bingham2019pyro} for posterior inference. M-HMC is used in \texttt{Tiger} to draw samples from the joint posterior distribution of the discrete latent states and a continuous parameter $\theta$. Whereas, for \texttt{RiverSwim} and general POMDPs, the likelihood of $\tau_{1:K}$ can be computed efficiently according to Eq. \ref{eq:likelihood_computation}, which allows us to use NUTS and directly sample from the posterior over the (continuous) parameter space (or the space of all transition and observation kernels).
Given a sampled POMDP, we use SARSOP \citep{kurniawati2009sarsop} to obtain a near-optimal policy. Our implementation is available at \url{https://github.com/dongzeye/ps4pomdp}.

\section{Conclusions}\label{sec:conclusions}

We introduced an online learning algorithm for POMDPs. We also introduced such an algorithm for multi-agent POMDPs in Appendix \ref{app:mapomdps}. Until not too long ago, POMDP problems were regarded as largely intractable. Currently, a number of effective approximate POMDP solvers are available when the model is known \citep{shani2013survey}. When the model is unknown, this however has remained a very challenging problem. Over the past couple of years, starting with \cite{jin2020sample}, a number of papers have proposed online learning algorithms for POMDPs with weaker assumptions. Unfortunately, however, most of them are primarily theoretical papers proposing conceptual algorithms that are hard to make practical. 

We proposed a simple posterior sampling-based algorithm that is implementable and computationally practical by use of a POMDP solver. While theoretical guarantees may not always be available, many of the available POMDP solvers yield good approximate solutions, and we can account for approximation errors via an additional term in our regret bound (Corollary \ref{cor:approx}). Our algorithm also requires computation of posterior updates. 
While a number of methods for tractable approximate posterior updates are available \citep{russo2018tutorial}, we were able to use a direct Bayesian inference algorithm for approximate posterior update. Our empirical results on a simple example show that our algorithm has surprisingly good performance (sublinear empirical regret) despite the approximations involved. 

\rjedit{While we have presented one of the first computationally implementable algorithms for online POMDP learning, more work needs to be done to make it  scalable to large POMDP settings. It would also be useful to extend the algorithm to infinite state space settings by using deep learning methods. These are important directions for future work. 
}

\bibliographystyle{abbrvnat}
\bibliography{mybib}

\newpage
\appendix

\section{Index Change Results}\label{app:changeofindex}
In this section we prove the two index change lemmas we used for the proof of regret upper bounds.

\subsection{Proof of Lemma \ref{lem:exchangeindexlitemain}}

We start from the following version of the elliptical potential lemma presented in \cite{carpentier2020elliptical}.

\begin{lemma}[Elliptical Potential Lemma \citep{carpentier2020elliptical}]\label{lem:epl}
    Let $x_k\in\mathbb{R}^d$ for all $k\in [K]$ be such that $\|x_k\|_2\leq 1$. Let $\lambda > 0$. Set $\mathbf{V}_k = \lambda \mathbf{I} + \sum_{j=1}^{k} x_j x_j^T$. Then, 
$\sum_{k=1}^K \sqrt{x_k^T \mathbf{V}_k^{-1} x_k} \leq \sqrt{dK \log\left(1 + \frac{K}{d\lambda}\right)}$.
\end{lemma}

Through rescaling, without loss of generality, assume that $\|w_k\|_2\leq G_wG_x$ and $\|x_k\|_2\leq 1$. Let $\lambda > 0$. Define $\mathbf{V}_k = \frac{\lambda}{G_w^2G_x^2} \mathbf{I} + \sum_{j=1}^{k} x_j x_j^T$. Let $\mathbf{V}_k^{1/2}$ be the symmetric square root matrix of the positive definite matrix $\mathbf{V}_k$. Using the elliptical potential lemma, we have
    \begin{align}
        &\quad~\sum_{k=1}^K |x_k^T w_k| = \sum_{k=1}^K |(\mathbf{V}_k^{-1/2} x_k)^T \mathbf{V}_k^{1/2} w_k|\\
        &\leq \sum_{k=1}^K \|\mathbf{V}_k^{-1/2} x_k\|_2\cdot \|\mathbf{V}_k^{1/2} w_k\|_2\\
        &= \sum_{k=1}^K \sqrt{x_k^T \mathbf{V}_k^{-1} x_k} \sqrt{w_k^T \mathbf{V}_k w_k}\\
        &\leq \sqrt{dK \log\left(1 + \dfrac{KG_w^2G_x^2}{d\lambda}\right)} \max_{k\in [K]} \sqrt{w_k^T \mathbf{V}_k w_k}
    \end{align}
    For each $k\in [K]$, we have
    \begin{align}
        w_k^T \mathbf{V}_k w_k &= \frac{\lambda}{G_x^2G_w^2} \|w_k\|_2^2 + \sum_{j=1}^{k}(x_j^T w_k)^2 \leq \lambda + \beta. 
    \end{align}
    Therefore, we conclude that
    \begin{align}
        \sum_{k=1}^K |x_k^T w_k| &\leq \sqrt{\left( \lambda  + \beta \right) dK \log\left(1 + \dfrac{KG_w^2G_x^2}{d\lambda}\right)}
    \end{align}
    for any $\lambda > 0$.

\subsection{An Important Technical Result for Theorem \ref{thm:breg1}}

The following index change result is key to improve the regret bound of \cite{liu2022partially}.

\begin{proposition}\label{lem:exchangeindexdeluxe}
    Let $w_{k, l, m}, x_{k, l, n}\in\mathbb{R}^d$ for $k\in [K], l\in [L], m\in [M], n\in [N]$. Suppose that
    \begin{align}
        &\sum_{j=1}^{k} \left(\sum_{(l, m, n)\in [L]\times [M]\times [N]} |w_{k, l, m}^T x_{j, l, n}|\right)^2\leq \beta,\\
        &\sum_{l=1}^L\sum_{m=1}^M\|w_{k, l, m}\|_1\leq G_w,\\
        &\sum_{l=1}^L\sum_{n=1}^N\|x_{k, l, n}\|_1\leq G_x,
    \end{align}
    for all $k\in [K]$. Then for any $\lambda > 0$,
    \begin{align}
        &\quad~\sum_{k=1}^K  \sum_{(l, m, n)\in [L]\times [M]\times [N]} |w_{k, l, m}^T x_{k, l, n}|\\
        &\leq \sqrt{\left( \lambda + \beta \right) dLMK \log\left(1 + \dfrac{MG_w^2G_x^2K}{dL\lambda}\right)}\label{eq:exchangeindexdeluxe}
    \end{align}
\end{proposition}

\begin{remark}
    The RHS of \eqref{eq:exchangeindexdeluxe} is tighter than the bound of Proposition 22 by \cite{liu2022partially} (which plays a similar role in their analysis) in several ways: (1)
    The bound by \cite{liu2022partially} has a $\log^2(KN)$ term. 
    In the context of its application in the main theorem, this term results in an extra $\Omega(H^2)$ term, since $N$ is the number of partial trajectories, which is exponential in $H$. (2) As \cite{liu2022partially} also pointed out, the bound by \cite{liu2022partially} depends linearly on $d$ instead of $\sqrt{d}$. (3) The bound by \cite{liu2022partially} scales linearly with $G_xG_w$, while in our bound the dependence is logarithmic. 
\end{remark}

\begin{proof}[Proof of Proposition \ref{lem:exchangeindexdeluxe}]
    Define $I_{k, l, m, n} := \mathrm{sgn}(w_{k, l, m}^T x_{k, l, n})$ and $ \Bar{x}_{k, l, m} := \sum_{n=1}^N x_{k, l, n} I_{k, l, m, n}$, we can then write the LHS of \eqref{eq:exchangeindexdeluxe} as $\sum_{k=1}^K \sum_{(l, m)\in [L]\times [M] } w_{k, l, m}^T \Bar{x}_{k, l, m}$.

    Let $w_k = (w_{k, l, m})_{(l, m)\in [L]\times [M]}$ and $\Bar{x}_k = (\Bar{x}_{k, l, m})_{(l, m)\in [L]\times [M]}$ be seen as $(dLM)$-dimensional vectors. 
    First, we have
    \begin{align}
        &\quad~\sum_{j=1}^{k} (w_k^T \Bar{x}_j)^2\\
        &= \sum_{j=1}^{k} \left(\sum_{(l, m)\in [L]\times [M] }  w_{k, l, m}^T \Bar{x}_{j, l, m} \right)^2\\
        &= \sum_{j=1}^{k} \left(\sum_{(l, m, n)\in [L]\times [M]\times [N] }  w_{k, l, m}^T x_{j, l, n} I_{j, l, m, n} \right)^2\\
        &\leq \sum_{j=1}^{k} \left(\sum_{(l, m, n)\in [L]\times [M]\times [N] }  |w_{k, l, m}^T x_{j, l, n}| \right)^2\leq \beta.
    \end{align}

    We also have
    \begin{align}
        \|w_k\|_2&\leq \|w_k\|_1 = \sum_{l=1}^L\sum_{m=1}^M\|w_{k, l, m}\|_1 \leq G_w,\\
        \|\Bar{x}_k\|_2&\leq \|\Bar{x}_k\|_1 = \sum_{l=1}^L\sum_{m=1}^M\|\Bar{x}_{k, l, m}\|_1\\
        &\leq \sum_{l=1}^L\sum_{m=1}^M\sum_{n=1}^N \|x_{k, l, n}\|_1 \leq M G_x.
    \end{align}

    The result is then established by applying Lemma \ref{lem:exchangeindexlitemain} on the $(dLM)$-dimensional vectors $w_k$ and $\Bar{x}_k$.
\end{proof}

\section{Proof of Proposition \ref{thm:breglower}}
\label{app:A}

We construct the random POMDP instance as follows: Let $\epsilon\in (0, 1)$ be determined later and let $\cS = \cO = \{0, 1\}$, where $0$ represents the ``good'' state and $1$ represents the ``bad'' state. Let the parameter space $\varTheta = \cA^{H-1}$ be the space of $(H-1)$-sequences of actions. For each $\theta = (\theta_{1:H-1})\in\varTheta$, the final state is ``good'' only if the agent follows the action sequence $\theta_{1:H-1}$, i.e., 
    \begin{align}
        b_0^{\theta} &= \delta_0\\
        T_h^{\theta}(s, a) &= \begin{cases}
            \delta_0&\text{if }s=0 \text{ and } a=\theta_h\\
            \delta_1&\text{otherwise}
        \end{cases}
    \end{align}
    where $\delta_s\in\Delta(\cS)$ represents the deterministic distribution supported on state $s$.

The observation kernel $Z^{\theta}$ is the same for all $\theta\in\varTheta$: Observations are uniformly random for all underlying states at all times $h\in [H-1]$. At the final timestamp, the observation kernel $\mathbb{Z}_H^{\theta}:\cS\mapsto \Delta(\cO)$ is given by
\begin{align}
    Z_H^{\theta}(o|s) = \begin{cases}
        \frac{1}{2}+\epsilon & s=0, o=0\\
        \frac{1}{2}-\epsilon & s=0, o=1\\
        \frac{1}{2} & s=1
    \end{cases}
\end{align}

The reward function is given by $r_h \equiv 0$ for $h\in[H-1]$ and $r_H(o, a) = \bm{1}_{\{o=0\}}$. The prior distribution $\nu^1$ is the uniform distribution on $\varTheta$.

For the convenience of the proof, we define an additional POMDP $(b_1^{\mathbf{u}}, T^{\mathbf{u}}, Z^{\mathbf{u}})$ by 
\begin{align}
    b_1^{\mathbf{u}} &= \delta_0\\
    T_h^{\mathbf{u}}(s, a) &= \delta_0\quad\forall s\in\cS,a\in\cS\\
    Z_h^{\mathbf{u}}(o|s) &= \frac{1}{2} \quad\forall o\in\cO, s\in\cS 
\end{align}

For any policy $\pi\in\varPi$ and any trajectory $\tau = (o_{1:H}, a_{1:H})\in \cT$, the trajectory probability satisfies
\begin{align}
    \Pr_{\theta}^{\pi}(\tau) &= \pi(\tau)\cdot \dfrac{1}{2^{H-1}} \left(\frac{1}{2} + \epsilon \iota_{\theta}(a_{1:H-1}, o_H) \right)\label{eq:lb:Pthetapi} \\
    \Pr_{\mathbf{u}}^{\pi}(\tau) &= \pi(\tau)\cdot \dfrac{1}{2^{H}}\label{eq:lb:Pupi}
\end{align}
where
\begin{align}
    \iota_{\theta}(a_{1:H-1}, o_H) := \begin{cases}
        1&\text{if }a_{1:H-1} = \theta\text{ and }o_H = 0\\
        -1&\text{if }a_{1:H-1} = \theta\text{ and }o_H = 1\\
        0&\text{if }a_{1:H-1} \neq \theta
    \end{cases}
\end{align}

Since the learning agent has perfect recall, the minimum Bayesian regret (equivalently, the maximum total expected reward) in $K$ episodes can be attained through some deterministic learning algorithm $\phi$. Therefore, without loss of generality, we only consider deterministic learning algorithms. More specifically, we assume that the learning algorithm is given by $\phi=(\phi_j)_{j\in\mathbb{N}}$ and $\phi_j$ maps the trajectories in the first $j-1$ episodes to a policy in $\varPi$. Fix $\phi$, let $\Pr_{\theta}$ and $\E_{\theta}$ represent the probability distribution and expectation of random variables involving (potentially) multiple episodes, assuming that the underlying POMDP is given by $\theta$, and the learning algorithm $\phi$ is applied. Similarly, let $\Pr_{\mathbf{u}}$ and $\E_{\mathbf{u}}$ represent the probability distribution and expectation of random variables under the POMDP defined by $(b_1^{\mathbf{u}}, T^{\mathbf{u}}, Z^{\mathbf{u}})$ and learning algorithm $\phi$.

For each $\theta\in\varTheta$, define $N_{\theta, K}: = \sum_{j=1}^K \bm{1}_{\{a_{1:H-1}^j = \theta\} } $ to be a random variable denoting the number of episodes in which the learning agent's action sequence matches $\theta$.

\begin{lemma}\label{lem:gainundertheta}
    For any $\theta\in\varTheta$, we have
    \begin{equation}
        \E_{\theta}[N_{\theta, K}]\leq \E_{\mathbf{u}}[N_{\theta, K}] + \frac{K}{2}\sqrt{-\E_{\mathbf{u}}[N_{\theta, K}] \log(1-4\epsilon^2)}
    \end{equation}
\end{lemma}

\begin{proof}

    \newcommand{\DKL}{\mathbf{D}_{\mathrm{KL}}}
    First, we have 
    \begin{align}
        &\quad~\E_{\theta}[N_{\theta, K}] - \E_{\mathbf{u}}[N_{\theta, K}]\\
        &\leq \sum_{\Tilde{\tau}^{1:K}: \Tilde{\tau}^j = (\Tilde{o}_{1:H}^j, \Tilde{a}_{1:H}^j ) \in \cT } \left[\sum_{j=1}^K \bm{1}_{\{\Tilde{a}_{1:H-1}^j = \theta\}} \right]  [\Pr_{\theta}(\Tilde{\tau}^{1:K}) - \Pr_{\mathbf{u}}( \Tilde{\tau}^{1:K})]_+\\
        &\leq K \|\Pr_{\mathbf{u}}(\tau^{1:K}) - \Pr_{\theta}(\tau^{1:K}) \|_{\mathrm{TV}}
    \end{align}

    By Pinsker's Inequality,
    \begin{align}
        \|\Pr_{\mathbf{u}}(\tau^{1:K}) - \Pr_{\theta}(\tau^{1:K}) \|_{\mathrm{TV}} \leq \sqrt{\dfrac{1}{2}\DKL(\Pr_{\mathbf{u}}(\tau^{1:K})\| \Pr_{\theta}(\tau^{1:K}) ) }
    \end{align}
    where $\DKL(\mu_1\|\mu_2)$ is the KL divergence between distributions $\mu_1$ and $\mu_2$.
    Using the chain rule for the KL divergence, we have
    \begin{align}
        &\quad~\DKL(\Pr_{\mathbf{u}}(\tau^{1:K})\| \Pr_{\theta}(\tau^{1:K}) )\\
        &= \sum_{j=1}^K \DKL(\Pr_{\mathbf{u}}(\tau^j|\tau^{1:j-1})\| \Pr_{\theta}(\tau^j|\tau^{1:j-1}) )\label{eq:chainrulekl} \\
        &=\sum_{j=1}^K \sum_{\Tilde{\tau}^{1:j-1}\in \cT^{j-1}} \Pr_{\mathbf{u}}(\Tilde{\tau}^{1:j-1}) \DKL(\Pr_{\mathbf{u}}(\tau^j|\Tilde{\tau}^{1:j-1})\| \Pr_{\theta}(\tau^j|\Tilde{\tau}^{1:j-1}) )
    \end{align}
    where in \eqref{eq:chainrulekl}, $\DKL(\Pr_{\mathbf{u}}(\tau^j|\tau^{1:j-1})\| \Pr_{\theta}(\tau^j|\tau^{1:j-1}) )$ is the conditional KL divergence.
    

    For a fixed history $\Tilde{\tau}^{1:j-1}$ before episode $j$, let $\Tilde{\pi}^j = \phi_j(\Tilde{\tau}^{1:j-1})$ be the corresponding policy the learning algorithm uses in episode $j$. Using \eqref{eq:lb:Pthetapi} and \eqref{eq:lb:Pupi} we have
    \begin{align}
        &\quad~\DKL(\Pr_{\mathbf{u}}(\tau^j|\Tilde{\tau}^{1:j-1})\| \Pr_{\theta}(\tau^j|\Tilde{\tau}^{1:j-1}) )\\
        &= \sum_{\tilde{\tau}\in \cT } \Pr_{\mathbf{u}}^{\Tilde{\pi}^j}(\Tilde{\tau})\log\left(\dfrac{\Pr_{\mathbf{u}}^{\Tilde{\pi}^j}(\Tilde{\tau})}{\Pr_{\theta}^{\Tilde{\pi}^j}(\Tilde{\tau})}\right)\\
        &= \Pr_{\mathbf{u}}^{\Tilde{\pi}^j}(a_{1:H-1} = \theta, o_H = 0)\log\left(\dfrac{1}{1+2\epsilon}\right) + \Pr_{\mathbf{u}}^{\Tilde{\pi}^j}(a_{1:H-1} = \theta, o_H = 1)\log\left(\dfrac{1}{1-2\epsilon}\right)\\
        &= \dfrac{1}{2} \Pr_{\mathbf{u}}^{\Tilde{\pi}^j}(a_{1:H-1} = \theta) \left[\log\left(\dfrac{1}{1+2\epsilon}\right) + \log\left(\dfrac{1}{1-2\epsilon}\right) \right]\\
        &= \dfrac{1}{2}\Pr_{\mathbf{u}}^{\Tilde{\pi}^j}(a_{1:H-1} = \theta) [-\log(1-4\epsilon^2)]= \dfrac{1}{2}\Pr_{\mathbf{u}}(a_{1:H-1}^j = \theta|\Tilde{\tau}^{1:j-1}) [-\log(1-4\epsilon^2)].
    \end{align}



    Therefore,
    \begin{align}
        &\quad~\DKL(\Pr_{\mathbf{u}}(\tau^{1:K})\| \Pr_{\theta}(\tau^{1:K}) )\\
        &= \dfrac{1}{2}\sum_{j=1}^K \sum_{\Tilde{\tau}^{1:j-1}} \Pr_{\mathbf{u}}(\Tilde{\tau}^{1:j-1}) \Pr_{\mathbf{u}}(a_{1:H-1}^j = \theta|\Tilde{\tau}^{1:j-1}) \left[-\log(1-4\epsilon^2)\right]\\
        &= \dfrac{1}{2}\sum_{j=1}^K \Pr_{\mathbf{u}}(a_{1:H-1}^j = \theta) \left[-\log(1-4\epsilon^2)\right]= -\dfrac{1}{2}\E_{\mathbf{u}}[N_{\theta, K}] \log\left(1 - 4\epsilon^2 \right)
    \end{align}

    and finally,
    \begin{align}
        \E_{\theta}[N_{\theta, K}] - \E_{\mathbf{u}}[N_{\theta, K}] &\leq K\sqrt{\dfrac{1}{2} \DKL(\Pr_{\mathbf{u}}(\tau^{1:K})\| \Pr_{\theta}(\tau^{1:K}) )}\\
        &= \dfrac{K}{2}\sqrt{-\E_{\mathbf{u}}[N_{\theta, K}] \log\left(1 - 4\epsilon^2 \right) }
    \end{align}
    
\end{proof}

{The conditional expected total reward $G_{\theta, K}$ under the algorithm $\phi$, given that the underlying POMDP is $\theta$, satisfies}
\begin{align}
    G_{\theta, K}&:=\E_{\theta}\left[\sum_{j=1}^K \sum_{h=1}^H r_h(o_{h}^j, a_h^j) \right]= \sum_{j=1}^K \Pr_{\theta}(o_H^j = 0)\\
    &= \sum_{j=1}^K \left[\left(\frac{1}{2}+\epsilon\right)\Pr_{\theta}(a_{1:H-1}^j = \theta) + \frac{1}{2} \Pr_{\theta}(a_{1:H-1}^j \neq \theta)\right]\\
    &= \dfrac{K}{2} + \epsilon  \sum_{j=1}^K \Pr_{\theta}(a_{1:H-1}^j = \theta) = \frac{K}{2} + \epsilon \E_{\theta}[N_{\theta, K}]
\end{align}

Define $\Bar{N}_{\theta, K}:= \E_{\mathbf{u}}[N_{\theta, K}]$ to be a fixed function of $\theta$ and $K$. Using Lemma \ref{lem:gainundertheta}, we have
\begin{align}
    G_{\theta, K} &\leq \frac{K}{2} + \epsilon\left(\Bar{N}_{\theta, K} + \dfrac{K}{2}\sqrt{-\Bar{N}_{\theta, K} \log\left(1 - 4\epsilon^2 \right) }\right)
\end{align}

 If $\theta^*$ is randomly drawn from $\varTheta$ following $\nu^1$, the uniform distribution on $\varTheta$, we have 
 \begin{align}
     \E_{\theta^*\sim\nu^1}\left[\Bar{N}_{\theta^*, K}\right]&= \dfrac{1}{A^{H-1}}\sum_{\theta\in\varTheta} \Bar{N}_{\theta, K} = \dfrac{1}{A^{H-1}}\sum_{\theta\in\varTheta} \E_{\mathbf{u}}[N_{\theta, K}]\\
     &= \dfrac{1}{A^{H-1}} \E_{\mathbf{u}}\left[\sum_{\theta\in\varTheta} N_{\theta, K} \right]=\dfrac{K}{A^{H-1}},
 \end{align}
 and the expected total reward satisfies
\begin{align}
    &\quad~\E_{\theta^*\sim \nu^1}\left[G_{\theta^*, K}\right] \\
    &\leq \dfrac{K}{2} + \epsilon \left(\E_{\theta^*\sim\nu^1}[\Bar{N}_{\theta^*, K}] + \dfrac{K}{2}\E_{\theta^*\sim\nu^1}\left[\sqrt{\Bar{N}_{\theta^*, K}}\right] \sqrt{-\log(1-4\epsilon^2)}  \right)\\
    &\leq \dfrac{K}{2} + \epsilon \left(\E_{\theta^*\sim\nu^1}[\Bar{N}_{\theta^*, K}] + \dfrac{K}{2}\sqrt{\E_{\theta^*\sim\nu^1}\left[\Bar{N}_{\theta^*, K}\right]} \sqrt{-\log(1-4\epsilon^2)}  \right)\\
    &= \dfrac{K}{2} + \epsilon\left(\dfrac{K}{A^{H-1}} + \dfrac{K}{2} \sqrt{\dfrac{K}{A^{H-1}}} \sqrt{-\log(1-4\epsilon^2)}\right)
\end{align}

For each POMDP parameterized by $\theta\in\varTheta$, if one uses the open-loop policy $a_h = \theta_h$ for $h\in [H-1]$, then the expected total reward is $\frac{1}{2}+\epsilon$. This means that the optimal value satisfies $V_{\theta}^*\geq \frac{1}{2}+\epsilon$. 

Therefore, the Bayesian regret is lower bounded by
\begin{align}
    \mathrm{BReg}(\phi, K)  &\geq \left(\dfrac{1}{2}+\epsilon\right) K - \E_{\theta^*\sim \nu^1}\left[G_{\theta^*, K}\right]\\
    &\geq \epsilon \left(K- \dfrac{K}{A^{H-1}} - \dfrac{K}{2} \sqrt{\dfrac{K}{A^{H-1}}} \sqrt{-\log(1-4\epsilon^2)}\right)
\end{align}

Theorem \ref{thm:breglower} can then be established by choosing $\epsilon = \frac{1}{4}\sqrt{\frac{A^{H-1}}{K}}$ and the inequality $-\log(1-x) \leq 4\log\left(\frac{4}{3}\right) x$ for $x\in [0, \frac{1}{4}]$ in the same way as in \cite{auer2002nonstochastic}.

\section{Proof of Regret Upper Bounds}\label{app:regub}
In this section, we prove both Theorem \ref{thm:breg0} and Theorem \ref{thm:breg1}. Our proofs are inspired by the framework for establishing regret bounds for posterior sampling by \cite{russo2014learning} and the techniques of \cite{liu2022partially}. 

In this section, $\Pr_{\theta}^{\pi}$ is seen as a probability distribution on $\cT$, the set of observable trajectories.
Without loss of generality, we assume that the parameterization is direct, i.e., $\varTheta$ is a subset of $[0, 1]^{S + (H-1)S^2A + HSO }$, where the coordinates represent $(b_1, T, Z)$. 

While our parameter set $\varTheta$ could be uncountably infinite,  we construct a finite set $\Bar{\varTheta}$ that serves as a surrogate for $\varTheta$ for technical reasons.

\begin{proposition}\label{lem:barvartheta}
    For any $\epsilon > 0$ such that $\epsilon^{-1}\in\mathbb{N}$, there exists a finite set $\Bar{\varTheta}\subset [0, 1]^{S + (H-1)S^2A + HSO }$ and a measurable function $\iota: \varTheta\mapsto \Bar{\varTheta}$ such that
    \begin{equation}
        \log(|\Bar{\varTheta}|) \leq (HS^2 A + HSO)\log\left(\dfrac{\max\{S, O\}}{\epsilon} + 1\right)
    \end{equation}
    and the following holds: For any $\theta\in\varTheta$, $\theta' = \iota(\theta)$ satisfies
    \begin{align}
        &\|\Pr_{\theta'}^\pi - \Pr_{\theta}^\pi\|_{\mathrm{TV}}\leq 2H\epsilon\qquad \forall \pi\in\varPi \label{eq:tvclose}\\
        &\Pr_{\theta'}^{-}(\tau) \geq (1+\epsilon)^{-2H} \Pr_{\theta}^{-}(\tau) \qquad\forall \tau\in\cT \label{eq:likelihoodclose}
    \end{align}
\end{proposition}
Recall that $\Pr_{\theta}^-$ represents the open-loop probability and is defined in \eqref{eq:envpart}. The proof of Proposition \ref{lem:barvartheta} is presented in Appendix \ref{app:b.1}.
In the rest of the section, we set $\epsilon = 1/(2H K)$ and assume that $\Bar{\varTheta}$ and $\iota$ satisfy the condition of Proposition \ref{lem:barvartheta}. By removing points in $\Bar{\varTheta}$, without loss of generality, we assume that $\iota$ is surjective, i.e., for every $\Bar{\theta}\in\Bar{\varTheta}$ there exists $\theta\in\varTheta$ such that $\Bar{\theta} = \iota(\theta)$.

Define the log-likelihood function
\begin{equation}\label{eq:def:llh}
    \ell(\theta, \mathcal{D}_k) := \sum_{j=1}^{k-1} \log(\Pr_{\theta}^{\pi^j}(\tau^j))
\end{equation}
and the maximum log-likelihood function
\begin{equation}\label{eq:def:maxllh}
    \Bar{\ell}^*(\mathcal{D}_k) := \max_{\theta\in \Bar{\varTheta}}\ell(\theta, \mathcal{D}_k).
\end{equation}
Note that in \eqref{eq:def:maxllh} the maximum is taken over the finite set $\Bar{\varTheta}$, not $\varTheta$.

\textbf{Quick Fact 1}: If $\Bar{\theta} = \iota(\theta)$, then 
\begin{align}
    \ell(\Bar{\theta}, \mathcal{D}_k) &\geq \ell(\theta, \mathcal{D}_k) - 2(k-1)H\log(1+\epsilon)\\
    &\geq \ell(\theta, \mathcal{D}_k) - 2(k-1)H\epsilon \geq \ell(\theta, \mathcal{D}_k) - 1
\end{align}

\textbf{Quick Fact 2}: As a corollary to Quick Fact 1, for any $\theta\in\varTheta$, $\ell(\theta, \mathcal{D}_k) \leq \Bar{\ell}^*(\mathcal{D}_k) + 1$.


We now define the discrete confidence set 
\begin{equation}
    \Bar{\varTheta}(\mathcal{D}_k) := \{\Bar{\theta}\in \Bar{\varTheta} : \ell(\Bar{\theta}, \mathcal{D}_k) \geq \Bar{\ell}^*(\mathcal{D}_k) - \log(K|\Bar{\varTheta}|) - 1 \}.
\end{equation}

The following lemmas form the foundation of our proof. In Lemma \ref{lem:step0}, we bound the Bayesian regret with the quality of the confidence set in terms of the TV distance between distributions on trajectories. In Lemma \ref{lem:step1}, we provide certain guarantee on the quality of the confidence sets. 

\begin{lemma}[Lemma \ref{lem:step0main}, restated]\label{lem:step0}
    The Bayesian regret under PS4POMDPs algorithm can be bounded by
    \begin{equation}\label{eq:bregtv}
        \mathrm{BReg}(\phi^{\mathrm{PS4POMDPs}}, K) \leq 2H + H \E\left[\sum_{k=1}^K   \max_{\Bar{\theta} \in \Bar{\varTheta}(\mathcal{D}_k) }\|\Pr_{\Bar{\theta}}^{\Tilde{\pi}^k} - \Pr_{\theta^*}^{\Tilde{\pi}^k}\|_{\mathrm{TV}} \right]
    \end{equation}
    
\end{lemma}

\begin{lemma}[Lemma \ref{lem:step1main}, restated]\label{lem:step1}
    Under any learning algorithm $\phi$, with probability at least $1 - \frac{1}{K}$, 
    \begin{equation}\label{eq:estfinal}
    \begin{split}
        \max_{k\in [K]}\max_{\Bar{\theta}\in\Bar{\varTheta}(\mathcal{D}_k) } \sum_{j=1}^{k}  \|\Pr_{\Bar{\theta}}^{\pi^j} - \Pr_{\theta^*}^{\pi^j}\|_{\mathrm{TV}}^2 \leq  3\log(K|\Bar{\varTheta}|) + 3 
    \end{split}
    \end{equation}
\end{lemma}

The proofs of Lemmas \ref{lem:step0} and \ref{lem:step1} can be found in Appendices \ref{sec:step0} and \ref{sec:step1}.

Given the above lemmas, to obtain an upper bound on the Bayesian regret, the remaining task is purely algebraic and can be described as follows: 
Use an upper bound of
\begin{align}
    \sum_{j=1}^{k}\|\Pr_{\check{\theta}^k}^{\Tilde{\pi}^j} - \Pr_{\theta^*}^{\Tilde{\pi}^j}\|_{\mathrm{TV}}^2\tag{TR-MLE}\label{eq:MLE}
\end{align}
to derive an upper bound of
\begin{align}
    \sum_{j=1}^{K} \|\Pr_{\check{\theta}^j}^{\Tilde{\pi}^j} - \Pr_{\theta^*}^{\Tilde{\pi}^j}\|_{\mathrm{TV}}\tag{TR-REG}\label{eq:REG}
\end{align}
where $\check{\theta}^k\in \argmax_{\Bar{\theta} \in \Bar{\varTheta}(\mathcal{D}_k) }\|\Pr_{\Bar{\theta}}^{\Tilde{\pi}^k} - \Pr_{\theta^*}^{\Tilde{\pi}^k}\|_{\mathrm{TV}}$.

The key difference between \eqref{eq:MLE} and \eqref{eq:REG} lies in the indexing, where the former measures the difference of the \emph{latest} estimated environment with the true environment under historical policies, while the latter measures the cumulative differences derived from the estimated environment \emph{at the time}. 


\begin{figure}[!ht]
    \centering
    \begin{tikzpicture}[scale=0.35]
        \draw[-stealth] (-10, 0.1) node[anchor=east]{BReg($\phi, K$)} -- (-4, 0.1);
        \draw[-stealth,dashed] (-10, -0.1) -- (-4, -0.1);
        \draw (-7, 0.1) node[anchor=south]{Lemma \ref{lem:step0}};
        \draw[-stealth,dashed] (1, 0) node[anchor=east]{\eqref{eq:REG}} -- (7, 0) node[anchor=west]{\eqref{eq:MLE}};
        \draw[-stealth] (12, 0.1) -- (18, 0.1) node[anchor=west]{$\log(K|\Bar{\varTheta}|)$};
        \draw[-stealth,dashed] (12, -0.1) -- (18, -0.1);
        \draw (15, 0.1) node[anchor=south]{Lemma \ref{lem:step1}};
        \draw[-stealth] (0, 1) -- (0, 7);
        \draw[-stealth] (1, 8) node[anchor=east]{(OP-REG)} -- (7, 8) node[anchor=west]{(OP-MLE)};
        \draw[-stealth] (8, 7) -- (8, 1);
        \draw (0, 4) node[anchor=east]{Step 1};
        \draw (4, 8) node[anchor=south]{Step 3};
        \draw (8, 4) node[anchor=west]{Step 2};
        \draw (4, 0) node[anchor=south]{Lemma \ref{lem:exchangeindexlitemain}};
    \end{tikzpicture}
    \caption{Proof road map of Theorem \ref{thm:breg0} (dashed) and Theorem \ref{thm:breg1} (solid). For Theorem \ref{thm:breg1}, the bound is developed via an auxiliary quantity called the projected operator distances. In this diagram, $A\rightarrow B$ means ``A is upper bounded by some function of B.'' While ``TR'' stands for ``trajectory-based'', ``OP'' stands for ``operator-based.''}
    \label{fig:proofapproach}
\end{figure}

\subsection{Proof of Proposition \ref{lem:barvartheta}}
\label{app:b.1}

For any finite set $\mathcal{X}$, define $\Bar{\Delta}^{\epsilon}(\mathcal{X}) \subset \Delta(\mathcal{X})$ by
\begin{equation}
    \Bar{\Delta}^{\epsilon}(\mathcal{X}) = \left\{ \dfrac{v}{\|v\|_1} : \dfrac{v_x}{\epsilon / |\mathcal{X}|} \in \mathbb{Z}_+, v_x \leq 1 \quad\forall x\in\mathcal{X} \right\}.
\end{equation}
Then we have $|\Bar{\Delta}^{\epsilon}(\mathcal{X})| \leq \left(\frac{|\mathcal{X}|}{\epsilon} + 1\right)^{|\mathcal{X}|}$. Furthermore, for any $\mu\in\Delta(\mathcal{X})$, set 
\begin{equation}\label{eq:qtznrenorm}
    v = \dfrac{1}{|\mathcal{X}| / \epsilon} \lceil (|\mathcal{X}| / \epsilon) \mu \rceil,\qquad \Bar{\mu} = \dfrac{v}{\|v\|_1},
\end{equation}
we have $\|v - \mu\|_1 \leq \epsilon$ and
\begin{align}
    \|v\|_1&\leq \|\mu\|_1 + \|v - \mu\|_1 \leq 1 + \epsilon.
\end{align}

Therefore,
\begin{align}
    \|\mu - \Bar{\mu} \|_{\mathrm{TV}}&\leq \dfrac{1}{2}\|\mu - v \|_1 + \dfrac{1}{2}\|v - \Bar{\mu}\|_1 \leq \dfrac{1}{2}\epsilon + \dfrac{1}{2}\left(\|v\|_1 - 1 \right) \leq \epsilon\\
    \Bar{\mu} &\geq \dfrac{\mu}{\|v\|_1} \geq \dfrac{\mu}{1 + \epsilon}.\label{eq:llratiobound}
\end{align}

Now, define $\Bar{\varTheta} = \left[\Bar{\Delta}^{\epsilon}(\cS)\right]^{1 + (H-1)SA} \times \left[\Bar{\Delta}^{\epsilon}(\cO)\right]^{HS}$. We have
\begin{align}
    \log |\Bar{\varTheta}| &\leq [1 + (H-1)SA] S\log\left(\dfrac{S}{\epsilon} + 1\right) + HSO \log\left(\dfrac{O}{\epsilon} + 1\right)\\
    &\leq (HS^2 A + HSO)\log\left(\dfrac{\max\{S, O\}}{\epsilon} + 1\right)
    .
\end{align}

For any $\theta\in \varTheta$, let $\Bar{\theta} \in\Bar{\varTheta}$ be such that each component of $\Bar{\theta}$ is obtained by quantizing and renormalizing the corresponding component in $\theta$ according to \eqref{eq:qtznrenorm}. 
Then, for any $\tau = (o_h, a_h)_{h=1}^H\in\cT$, using \eqref{eq:envpart} and \eqref{eq:llratiobound} we have $\Pr_{\Bar{\theta}}^{-}(\tau) \geq (1+\epsilon)^{-2H} \Pr_{\theta}^{-}(\tau)$, establishing \eqref{eq:likelihoodclose}. 

Next we establish \eqref{eq:tvclose} for $\theta$ and $\Bar{\theta}$. Define a coupled POMDP with state space $\cS^2$, action space $\cA^2$, observation space $\cO^2$, horizon $H$ as follows
\begin{itemize}
    \item The initial distribution $\bm{b}_1^{\theta, \Bar{\theta}}$ is an optimal coupling\footnote{A coupling of two distributions $\nu_1, \nu_2\in\Delta(\varOmega)$ is a joint distribution $\nu_{12}\in\Delta(\varOmega\times\varOmega)$ with marginal distributions $\nu_1$ and $\nu_2$. An optimal coupling \citep{levin2017markov} is a coupling that attains the minimum value of the mismatch probability $\nu_{12}(\{(\omega_1, \omega_2): \omega_1\neq\omega_2 \})$, which equals $\|\nu_1 - \nu_2\|_{\mathrm{TV}}$. See \cite{levin2017markov} for a construction of such coupling.} of $b_1^{\theta}$ and $b_1^{\Bar{\theta}}$;
    \item The transition kernel $\mathbf{T}^{\theta, \Bar{\theta}}: (\cS\times\cA)^2 \mapsto \Delta(\cS^2)$ is such that for any $h\in [H-1], s, \Bar{s}\in\cS, a,\Bar{a}\in\cA$, the joint distribution $\mathbf{T}_h^{\theta, \Bar{\theta}}(s, a, \Bar{s}, \Bar{a})$ is an optimal coupling of $T_h^{\theta}(s, a)$ and $T_h^{\Bar{\theta}}(\Bar{s}, \Bar{a})$.
    \item The observation kernel $\mathbf{Z}^{\theta, \Bar{\theta}}: \cS^2 \mapsto \Delta(\cO^2)$ is such that for any $h\in [H], s, \Bar{s}\in\cS, o,\Bar{o}\in\cO$, the joint distribution $\mathbf{Z}_h^{\theta, \Bar{\theta}}(s, \Bar{s})$ is an optimal coupling of $Z^\theta_h(s)$ and $Z^{\Bar{\theta}}_h(\Bar{s})$.
\end{itemize}

In the coupled POMDP, at any given step, if the current states are the same for both components, then the current observation would be the same with probability at least $1 - \epsilon$; if the current state and action pairs are the same for both components, then the next states would be the same with probability at least $1 - \epsilon$. Consequently, if the same policy $\pi\in\varPi$ is applied on both components, then with probability at least $(1-\epsilon)^{2H}\geq 1 - 2H\epsilon$, the two trajectories $\tau, \Bar{\tau}\in \cT$ coincide. Due to the fact that the TV distance equals the minimum mismatch probability of any coupling \citep{levin2017markov}, we conclude that $\|\Pr_{\theta}^\pi - \Pr_{\Bar{\theta}}^\pi \|_{\mathrm{TV}}\leq 2H\epsilon$.

\subsection{Proof of Lemma \ref{lem:step0} (Lemma \ref{lem:step0main}, restated)}\label{sec:step0}

First, the episode $k$ regret satisfies
\begin{align}
    &\quad~\E[V_{\theta^*}^*] - \E[V_{\theta^*}^{\Tilde{\pi}^k}]\\
    &= \E[V_{\Tilde{\theta}^k}^*] - \E[V_{\theta^*}^{\Tilde{\pi}^k}] = \E[V_{\Tilde{\theta}^k}^{\Tilde{\pi}^k} - V_{\theta^*}^{\Tilde{\pi}^k}]\\
    &=\E\left[\sum_{\tau = (o_{h}, a_{h})_{h=1}^H\in\cT } [\Pr_{\Tilde{\theta}^k}^{\Tilde{\pi}^k}(\tau) - \Pr_{\theta^*}^{\Tilde{\pi}^k}(\tau)]\sum_{h=1}^H r_h(o_h, a_h) \right]\\
    &\leq H \E\left[\sum_{\tau\in\cT} [\Pr_{\Tilde{\theta}^k}^{\Tilde{\pi}^k}(\tau) - \Pr_{\theta^*}^{\Tilde{\pi}^k}(\tau) ]_+\right] = H \E[\|\Pr_{\Tilde{\theta}^k}^{\Tilde{\pi}^k} - \Pr_{\theta^*}^{\Tilde{\pi}^k} \|_{\mathrm{TV}}]
\end{align}
where in the first identity we used the fact that $\theta^*$ has the same unconditional distribution as $\Tilde{\theta}^k$.

Using Lemma \ref{lem:confset} and the fact that conditioning on $\mathcal{D}_k$, $\theta^*$ has the same distribution as $\Tilde{\theta}^k$, we have $\ell(\Tilde{\theta}^k, \mathcal{D}_k) \geq \Bar{\ell}^*(\mathcal{D}_k) - \log(K|\Bar{\varTheta}|)$  with probability at least $1-\frac{1}{K}$. Let $\vartheta^k = \iota\left(\Tilde{\theta}^k\right)$, we have $\ell(\vartheta^k, \mathcal{D}_k) \geq \ell(\Tilde{\theta}^k, \mathcal{D}_k) - 1$ almost surely due to Quick Fact 1. Therefore, we conclude that
\begin{align}
    \Pr\left(\vartheta^k \in \Bar{\varTheta}(\mathcal{D}_k)\right) &= \Pr\left(\ell(\vartheta^k, \mathcal{D}_k) \geq \Bar{\ell}^*(\mathcal{D}_k) - \log(K|\Bar{\varTheta}|) - 1 \right)  \geq 1-\dfrac{1}{K}
\end{align}

Also note that $\|\Pr_{\Tilde{\theta}^k}^{\Tilde{\pi}^k} - \Pr_{\vartheta^k}^{\Tilde{\pi}^k}\|_{\mathrm{TV}} \leq 2H\epsilon = \frac{1}{K}$ almost surely. Therefore
\begin{align}
    \E[\|\Pr_{\Tilde{\theta}^k}^{\Tilde{\pi}^k} - \Pr_{\theta^*}^{\Tilde{\pi}^k} \|_{\mathrm{TV}}] &\leq  \E[\|\Pr_{\vartheta^k}^{\Tilde{\pi}^k} - \Pr_{\theta^*}^{\Tilde{\pi}^k} \|_{\mathrm{TV}}] + \dfrac{1}{K}\\
    &\leq \E[\|\Pr_{\vartheta^k}^{\Tilde{\pi}^k} - \Pr_{\theta^*}^{\Tilde{\pi}^k} \|_{\mathrm{TV}}\cdot \bm{1}_{\{\vartheta^k \in \Bar{\varTheta}(\mathcal{D}_k) \}} ] + \Pr(\vartheta^k \not\in \Bar{\varTheta}(\mathcal{D}_k)) + \dfrac{1}{K}\\
    &\leq \E\left[\max_{\Bar{\theta}\in \Bar{\varTheta}(\mathcal{D}_k)} \|\Pr_{\Bar{\theta}}^{\Tilde{\pi}^k} - \Pr_{\theta^*}^{\Tilde{\pi}^k} \|_{\mathrm{TV}} \right] + \dfrac{2}{K}
\end{align}
and finally
\begin{align}
    \mathrm{BReg}(\phi^{\mathrm{PS4POMDPs}}, K) &\leq \sum_{k=1}^K H\left( \E\left[ \max_{\Bar{\theta}\in \Bar{\varTheta}(\mathcal{D}_k)} \|\Pr_{\Bar{\theta}}^{\Tilde{\pi}^k} - \Pr_{\theta^*}^{\Tilde{\pi}^k} \|_{\mathrm{TV}} \right] + \dfrac{2}{K}\right)\\
    &\leq 2H + H\E\left[\sum_{k=1}^K \max_{\Bar{\theta}\in \Bar{\varTheta}(\mathcal{D}_k)} \|\Pr_{\Bar{\theta}}^{\Tilde{\pi}^k} - \Pr_{\theta^*}^{\Tilde{\pi}^k} \|_{\mathrm{TV}} \right].
\end{align}

\subsection{Proof of Lemma \ref{lem:step1} (Lemma \ref{lem:step1main}, restated)}\label{sec:step1}

Consider the distribution on the learning process conditioning on $\theta^* = \theta$. Fix $\theta'\in\Bar{\varTheta}$. Define $\mathcal{F}_{k-1}$ to be the $\sigma$-algebra generated from $(\mathcal{D}_k, {\pi}^k )$ and $Z_k = \log\left(\dfrac{\Pr_{\theta'}^{{\pi}^k} (\tau^k) }{ \Pr_{\theta}^{{\pi}^k} (\tau^k)}\right) $. Then the quantity $\psi_k$ defined in Lemma \ref{lem:logeexp} satisfies
\begin{align}
    \psi_k\left(\frac{1}{2}\right) &= \log \E_{\theta}\left[\exp\left(\frac{1}{2} \log\left(\dfrac{\Pr_{\theta'}^{{\pi}^k} (\tau^k) ) }{ \Pr_{\theta}^{{\pi}^k} (\tau^k)}\right) \right)~\Big|~\mathcal{D}_k, {\pi}^k\right]\\
    &= \log \left( \sum_{\tau\in\cT:\Pr_{\theta}^{{\pi}^k} (\tau) > 0 } \Pr_{\theta}^{{\pi}^k} (\tau)\sqrt{\dfrac{\Pr_{\theta'}^{{\pi}^k} (\tau) }{\Pr_{\theta}^{{\pi}^k} (\tau) }} \right) = \log \left(\sum_{\tau\in\cT} \sqrt{\Pr_{\theta'}^{{\pi}^k} (\tau) \Pr_{\theta}^{{\pi}^k} (\tau)} \right)\\
    &\leq \left(\sum_{\tau\in\cT} \sqrt{\Pr_{\theta'}^{{\pi}^k} (\tau) \Pr_{\theta}^{{\pi}^k} (\tau)}\right) - 1= \sum_{\tau\in\cT} \left(\sqrt{\Pr_{\theta'}^{{\pi}^k} (\tau) \Pr_{\theta}^{{\pi}^k} (\tau)} - \frac{1}{2} \Pr_{\theta'}^{{\pi}^k} (\tau) - \frac{1}{2} \Pr_{\theta}^{{\pi}^k} (\tau) \right)\\
    &= -\dfrac{1}{2} \sum_{\tau\in\cT} \left(\sqrt{\Pr_{\theta'}^{{\pi}^k}(\tau)} - \sqrt{\Pr_{\theta}^{{\pi}^k}(\tau)} \right)^2 \leq -\dfrac{1}{2}\|\Pr_{\theta'}^{{\pi}^k} - \Pr_{\theta}^{{\pi}^k} \|_{\mathrm{TV}}^2
\end{align}
where we have used the inequality $\log(t) \leq t - 1$, and the last inequality is due to Lemma \ref{lem:hellinger}.

Applying Lemma \ref{lem:logeexp}, for any $\xi>0$, with probability at least $1-e^{-\xi}$, 
\begin{align}
    \sum_{j=1}^{k-1} \left(\frac{1}{2}Z_j +  \frac{1}{2} \|\Pr_{\theta'}^{{\pi}^j} - \Pr_{\theta}^{{\pi}^j} \|_{\mathrm{TV}}^2\right)   \leq \sum_{j=1}^{k-1} \left(\frac{1}{2}Z_j - \psi_j\left(\frac{1}{2}\right) \right) < \xi  \qquad\forall k\in [K]
\end{align}
or equivalently,
\begin{align}
    \sum_{j=1}^{k-1}  \|\Pr_{\theta'}^{{\pi}^j} - \Pr_{\theta}^{{\pi}^j} \|_{\mathrm{TV}}^2 < 2\xi - \sum_{j=1}^{k-1} Z_j = 2\xi + \ell(\theta, \mathcal{D}_{k}) - \ell(\theta', \mathcal{D}_{k}) \qquad\forall k\in [K]\label{eq:tvboundworkingprogress}
\end{align}

Using the fact that TV distances are bounded by 1, adding in the $k$-th term, \eqref{eq:tvboundworkingprogress} implies
\begin{align}
    \sum_{j=1}^{k}  \|\Pr_{\theta'}^{{\pi}^j} - \Pr_{\theta}^{{\pi}^j} \|_{\mathrm{TV}}^2 < 2\xi + 1 + \ell(\theta, \mathcal{D}_{k}) - \ell(\theta', \mathcal{D}_{k})\qquad \forall k\in [K]
\end{align}

Now, if $\theta'\in \Bar{\varTheta}(\mathcal{D}_k)$, then
\begin{align}
    &\quad~ \ell(\theta, \mathcal{D}_k) -\ell(\theta', \mathcal{D}_k) \leq \Bar{\ell}^*(\mathcal{D}_k) + 1 - \ell(\theta', \mathcal{D}_k) \leq \log(K|\Bar{\varTheta}|) + 2
\end{align}
where the first inequality is due to Quick Fact 2, and the second inequality is due to the definition of $\Bar{\varTheta}(\mathcal{D}_k)$. 

Finally, we have
\begin{align}
    &\quad~\Pr_{\theta}\left(\max_{k\in [K]} \max_{\theta'\in\Bar{\varTheta}(\mathcal{D}_k)} \sum_{j=1}^{k-1}  \|\Pr_{\theta'}^{{\pi}^j} - \Pr_{\theta}^{{\pi}^j} \|_{\mathrm{TV}}^2 \geq 3\log\left(K|\Bar{\varTheta}| \right) + 3\right)\\
    &= \Pr_{\theta}\left(\exists k\in [K], \theta'\in\Bar{\varTheta}(\mathcal{D}_k),~~ \sum_{j=1}^{k-1}  \|\Pr_{\theta'}^{{\pi}^j} - \Pr_{\theta}^{{\pi}^j} \|_{\mathrm{TV}}^2 \geq 3\log\left(K|\Bar{\varTheta}| \right) + 3\right)\\
    &\leq \Pr_{\theta}\left(\exists k\in [K], \theta'\in\Bar{\varTheta},~~ \sum_{j=1}^{k-1}  \|\Pr_{\theta'}^{{\pi}^j} - \Pr_{\theta}^{{\pi}^j} \|_{\mathrm{TV}}^2 \geq 2\log\left(K|\Bar{\varTheta}| \right) + 1 + \ell(\theta, \mathcal{D}_{k}) - \ell(\theta', \mathcal{D}_{k}) \right)\\
    &\leq \sum_{\theta'\in\Bar{\varTheta}} \Pr_{\theta}\left(\exists k\in [K],~~ \sum_{j=1}^{k-1}  \|\Pr_{\theta'}^{{\pi}^j} - \Pr_{\theta}^{{\pi}^j} \|_{\mathrm{TV}}^2 \geq 2\log\left(K|\Bar{\varTheta}| \right) + 1 + \ell(\theta, \mathcal{D}_{k}) - \ell(\theta', \mathcal{D}_{k}) \right) \\
    &\leq \sum_{\theta'\in\Bar{\varTheta}} \exp\left(-\log(K|\Bar{\varTheta}|)\right) = \dfrac{1}{K}.
\end{align}

The above holds for any $\theta\in\varTheta$. Therefore, we conclude that with probability at least $1-\frac{1}{K}$,
\begin{align}
    \max_{k\in [K]}\max_{\theta'\in\Bar{\varTheta}(\mathcal{D}_k)} \sum_{j=1}^{k}  \|\Pr_{\theta'}^{{\pi}^j} - \Pr_{\theta^*}^{{\pi}^j} \|_{\mathrm{TV}}^2 < 3\log\left(K|\Bar{\varTheta}|\right) + 3.
\end{align}

\subsection{Proof of Theorem \ref{thm:breg0}}\label{app:breg0}

Define $(OA)^H$-dimensional vectors $x_k, w_k$ via
\begin{align}
    x_k(\tau) &:= \Tilde{\pi}^k(\tau)\\
    w_k(\tau) &:= \frac{1}{2}|\Pr_{\check{\theta}^k}^{-}(\tau) - \Pr_{\theta^*}^{-}(\tau)|
\end{align}
for each trajectory $\tau\in \cT$, then \eqref{eq:MLE} can be written as $\sum_{j=1}^k (w_k^T x_j)^2$ and \eqref{eq:REG} can be written as $\sum_{k=1}^K w_k^T x_k$. We have
\begin{align}
    \|x_k\|_2 &\leq \|x_k\|_1 = \sum_{\tau\in \cT } \Tilde{\pi}^k(\tau) =  O^H\\
    \|w_k\|_2 &\leq \|w_k\|_1 = \sum_{\tau\in \cT } \frac{1}{2}|\Pr_{\check{\theta}^k}^{-}(\tau) - \Pr_{\theta^*}^{-}(\tau)|= A^H \sum_{\tau\in \cT } \frac{1}{2}|\Pr_{\check{\theta}^k}^{\mathbf{unif}}(\tau) - \Pr_{\theta^*}^{\mathbf{unif}}(\tau)| \leq A^H 
\end{align}
where $\mathbf{unif}$ represents the policy of picking a uniform random action at all times.

Lemma \ref{lem:step1}, in particular, implies that $\sum_{j=1}^k (w_k^T x_j)^2\leq 3\log(K|\Bar{\varTheta}|) + 3$ with probability at least $1-\frac{1}{K}$. 

Applying Lemma \ref{lem:exchangeindexlitemain} with $\lambda = 1$, with probability $1-\frac{1}{K}$, we obtain
\begin{equation}
    \eqref{eq:REG} \leq \sqrt{(1 + \beta)(OA)^H K \log\left[1 + (OA)^{H}K\right] } 
\end{equation}
where $\beta := 3\log(K|\Bar{\varTheta}|) + 3 = \bigOtilde(HS^2A + HSO)$.

Since TV distances are bounded by 1, we also have $\eqref{eq:REG}\leq K$ with probability 1, hence 
\begin{align}
    \E[\eqref{eq:REG}] &\leq \sqrt{(1 + \beta)(OA)^H K \log\left[1 + (OA)^{H}K\right] } + 1
\end{align}

Finally, using Lemma \ref{lem:step0}, we conclude that
\begin{align}
    \mathrm{BReg}(\phi^{\mathrm{PS4POMDPs}}, K) &\leq 2H + H \E\left[\eqref{eq:REG}\right] \\
    &\leq 2H + H\sqrt{(1 + \beta)(OA)^H K \log\left[1 + (OA)^{H}K\right] } + H \\
    &= \bigOtilde\left(H\sqrt{H\beta  (OA)^H K}\right) = \bigOtilde\left(H^2\sqrt{(S^2A + SO) (OA)^H K}\right). 
\end{align}

\subsection{Proof of Theorem \ref{thm:breg1}}\label{app:breg1}

To obtain a bound polynomial in $H$, we will follow the indirect approach as illustrated in Figure \ref{fig:proofapproach} and incorporate Assumption \ref{assump:obsfullrank} along the way. The steps here follow closely that of \cite{liu2022partially}. However, through the use of an index change lemma with a tighter upper bound guarantee (our Proposition \ref{lem:exchangeindexdeluxe} in Appendix \ref{app:changeofindex}, compared to Proposition 22 by \cite{liu2022partially}), we are able to establish a regret bound for weakly $\alpha$-revealing POMDPs with improved dependency on $H, S,$ and $A$. 

Before we start the proof, we first establish a few quick results from Assumption \ref{assump:obsfullrank}:

\textbf{Quick Fact 3:} $\alpha \leq \sqrt{S}$. 

\begin{proof}
We have $\sigma_{\min}(\mathbb{Z}_h^{\theta}) \leq \|\mathbb{Z}_h^{\theta}\|_2  \leq \sqrt{S} \|\mathbb{Z}_h^{\theta}\|_1$. Furthermore, $\|\mathbb{Z}_h^{\theta}\|_1 = \max_{s\in\cS}\sum_{o\in\cO} |\mathbb{Z}_h^{\theta}(o, s)| = 1$.
\end{proof}

\textbf{Quick Fact 4:} Let $\mathbb{Z}_h^{\theta\dagger}$ be the Moore-Penrose inverse of $\mathbb{Z}_h^{\theta}$. We have $\|\mathbb{Z}_h^{\theta\dagger} \|_1 \leq \dfrac{\sqrt{S}}{\alpha}$ for $\theta\in \varTheta$.

\begin{proof} 
    $\|\mathbb{Z}_h^{\theta\dagger}\|_1 \leq \sqrt{S}\|\mathbb{Z}_h^{\theta\dagger}\|_2 = \sqrt{S}\sigma_{\max}(\mathbb{Z}_h^{\theta\dagger}) = \dfrac{\sqrt{S}}{\sigma_{\min}(\mathbb{Z}_h^{\theta})} \leq \dfrac{\sqrt{S}}{\alpha}$.
\end{proof}

Similarly to the proof of Theorem \ref{thm:breg0}, in this proof, we attempt to bound \eqref{eq:REG} with \eqref{eq:MLE}. Note that while Assumption \ref{assump:obsfullrank} is imposed on the set $\varTheta$, both \eqref{eq:REG} and \eqref{eq:MLE} involve some $\Bar{\theta}\in \Bar{\varTheta}$. The fact that $\Bar{\varTheta}$ is not a subset of $\varTheta$ prevents us from applying Assumption \ref{assump:obsfullrank}. Therefore, we need the following step. Recall that we have assumed that the quantization map $\iota$ is surjective. Let $\hat{\theta}^k\in \varTheta$ be such that $\check{\theta}^k = \iota(\hat{\theta}^k)$, then we have $\|\Pr_{\check{\theta}^k}^\pi - \Pr_{\hat{\theta}^k}^\pi \|_{\mathrm{TV}} \leq 2H\epsilon = \frac{1}{K}$. Therefore, we have
\begin{align}
    \sum_{j=1}^{k}\|\Pr_{\hat{\theta}^k}^{\Tilde{\pi}^j} - \Pr_{\theta^*}^{\Tilde{\pi}^j}\|_{\mathrm{TV}}^2 &\leq \sum_{j=1}^{k}\left(\|\Pr_{\check{\theta}^k}^{\Tilde{\pi}^j} - \Pr_{\theta^*}^{\Tilde{\pi}^j}\|_{\mathrm{TV}} + \dfrac{1}{K} \right)^2 = \sum_{j=1}^k \left(\|\Pr_{\check{\theta}^k}^{\Tilde{\pi}^j} - \Pr_{\theta^*}^{\Tilde{\pi}^j}\|_{\mathrm{TV}}^2 + \dfrac{2}{K}\|\Pr_{\check{\theta}^k}^{\Tilde{\pi}^j} - \Pr_{\theta^*}^{\Tilde{\pi}^j}\|_{\mathrm{TV}} + \dfrac{1}{K^2}\right) \\
    &\leq \left(\sum_{j=1}^k \|\Pr_{\check{\theta}^k}^{\Tilde{\pi}^j} - \Pr_{\theta^*}^{\Tilde{\pi}^j}\|_{\mathrm{TV}}^2 \right) + k\left(\dfrac{2}{K} + \dfrac{1}{K^2}\right)  \leq \eqref{eq:MLE} + 2 + \dfrac{1}{K}\label{eq:mleleqmle}
\end{align}
and
\begin{align}
    \eqref{eq:REG} \leq \sum_{j=1}^{K} \left(\|\Pr_{\hat{\theta}^j}^{\Tilde{\pi}^j} - \Pr_{\theta^*}^{\Tilde{\pi}^j}\|_{\mathrm{TV}} + \dfrac{1}{K}\right) \leq \left(\sum_{j=1}^K \|\Pr_{\hat{\theta}^j}^{\Tilde{\pi}^j} - \Pr_{\theta^*}^{\Tilde{\pi}^j}\|_{\mathrm{TV}}\right) + 1\label{eq:regleqreg}.
\end{align}

Therefore, our task for this proof is to use
\begin{align}
    \sum_{j=1}^{k}\|\Pr_{\hat{\theta}^k}^{\Tilde{\pi}^j} - \Pr_{\theta^*}^{\Tilde{\pi}^j}\|_{\mathrm{TV}}^2\tag{TR-MLE-II}\label{eq:MLE2}
\end{align}
to bound
\begin{align}
    \sum_{j=1}^K \|\Pr_{\hat{\theta}^j}^{\Tilde{\pi}^j} - \Pr_{\theta^*}^{\Tilde{\pi}^j}\|_{\mathrm{TV}}\tag{TR-REG-II}\label{eq:REG2}.
\end{align}

\textbf{Step 1: Bound \eqref{eq:REG2} with Projected Distances of Observable Operators.}

Recall that the probability of a trajectory of a POMDP can be written with matrix multiplications. Recall that $\mathbb{T}_{h, a}\in \mathbb{R}^{\cS\times\cS}$ is the probability transition matrix (where the rows represent the next state) under action $a\in\cA$ at time $h\in [H-1]$. Recall that $\mathbb{Z}_h \in \mathbb{R}^{\cO\times \cS}$ is the observation probability matrix. For each $h\in [H-1], a\in\cA, o\in\cO$, define the \emph{observable operator} matrix \citep{jin2020sample}
\begin{equation}\label{eq:defoop}
    \mathbf{B}_h^{\theta}(a, o) := 
    \mathbb{Z}_{h+1}^{\theta} \mathbb{T}_{h, a}^{\theta}\mathrm{diag}(\mathbb{Z}_h^{\theta}(o, \cdot)) \mathbb{Z}_h^{\theta\dagger}
\end{equation}
where $\mathbb{Z}_h^{\theta\dagger}$ is the Moore-Penrose inverse of $\mathbb{Z}_h^{\theta}$. Note that due to Assumption \ref{assump:obsfullrank}, $\mathbb{Z}_h^{\theta}$ has full column rank, hence $\mathbb{Z}_h^{\theta\dagger}\mathbb{Z}_h^\theta = \mathbf{I}$. We can rewrite the environment part of the probability of a trajectory as
\begin{equation}\label{eq:operatorpomdp}
    \Pr_\theta^{-}(o_{1:H}, a_{1:H}) = \mathbf{e}_{o_H}^T \mathbf{B}_{H-1}^{\theta}(a_{H-1}, o_{H-1}) \cdots \mathbf{B}_{1}^{\theta}(a_{1}, o_{1}) \mathbb{Z}_1^{\theta} b_1^{\theta}
\end{equation}

For a full trajectory $\tau = (o_{1:H}, a_{1:H})\in\cT$, let $\tau_h = (o_{1:h}, a_{1:h})$ denote its corresponding partial trajectory up to time $h$ and $\tau_{-h} = (o_{h+1:H}, a_{h+1:H})$ denote its corresponding partial trajectory from time $h+1$ to $H$. Define $\mathbf{v}_0^{\theta} = \mathbb{Z}_1^\theta b_1^\theta$ and $\mathbf{v}_h^\theta(\tau_h) := \mathbf{B}_h^\theta(a_h, o_h)\cdots \mathbf{B}_1^\theta(a_1, o_1) \mathbf{v}_0^{\theta}$. With some abuse of notation, define 
\begin{equation}\label{eq:defBHh}
    \mathbf{B}_{H:h}^{\theta}(\tau_{-(h-1)}) := 
    \begin{cases}
        \mathbf{e}_{o_H}^T&h=H\\
        \mathbf{e}_{o_H}^T\mathbf{B}_{H-1}^{\theta}(a_{H-1}, o_{H-1})\cdots \mathbf{B}_{h}^{\theta}(a_{h}, o_{h})&h<H
    \end{cases}
\end{equation}

In the rest of the step, we will prove the following claim.

\textbf{Claim}: With probability 1, we have $\eqref{eq:REG2} \leq \dfrac{\sqrt{S}}{\alpha}\sum_{h=0}^{H-1}(\text{OP-REG-}h)$, where
\begin{align}
    (\text{OP-REG-}0):= &\dfrac{1}{2}\sum_{j=1}^{K} \|\mathbf{v}_0^{\theta^*} - \mathbf{v}_0^{\hat{\theta}^j}\|_1\label{eq:opREG0}\\
    (\text{OP-REG-}h):=&\dfrac{1}{2}\sum_{j=1}^{K} \sum_{\tau_h\in (\cO\times\cA)^h } \Tilde{\pi}^j(\tau_h) \|[\mathbf{B}_{h}^{\theta^*}(a_{h}, o_{h}) - \mathbf{B}_{h}^{\hat{\theta}^j}(a_{h}, o_{h})]\mathbf{v}_{h-1}^{\theta^*}(\tau_{h-1})\|_1\qquad\forall h\in [H-1] \label{eq:opREGh}
\end{align}
are the \emph{projected operator distances}.\\~

From \eqref{eq:operatorpomdp}, using telescoping sums, for any $\pi\in\varPi$ and parameters $\theta, \theta'$ we have
\begin{align}
    &\quad~ \|\Pr_{\theta}^\pi - \Pr_{\theta'}^{\pi} \|_{\mathrm{TV}}= \dfrac{1}{2} \sum_{\tau\in\cT} \pi(\tau) |\Pr_\theta^{-}(\tau) -  \Pr_{\theta'}^{-}(\tau)| \\
    &\leq \dfrac{1}{2}\sum_{\tau\in\cT} \pi(\tau) |\mathbf{B}_{H:1}^{\theta'}(\tau)[\mathbf{v}_0^{\theta} - \mathbf{v}_0^{\theta'}] | + \\
    &+\dfrac{1}{2}\sum_{\tau\in\cT}\pi(\tau)\sum_{h=1}^{H-1} |\mathbf{B}_{H:h+1}^{\theta'}(\tau_{-h})[\mathbf{B}_{h}^{\theta}(a_{h}, o_{h}) - \mathbf{B}_{h}^{\theta'}(a_{h}, o_{h})]\mathbf{v}_{h-1}^{\theta}(\tau_{h-1}) |
\end{align}

Using Lemma \ref{lem:BHhbound}, we have
\begin{align}
    \sum_{\tau\in\cT} \pi(\tau) |\mathbf{B}_{H:1}^{\theta'}(\tau)[\mathbf{v}_0^{\theta} - \mathbf{v}_0^{\theta'}] | \leq \dfrac{\sqrt{S}}{\alpha}\|\mathbf{v}_0^{\theta} - \mathbf{v}_0^{\theta'}\|_1 
\end{align}

Also through Lemma \ref{lem:BHhbound}, for each $h \in [H-1]$ we have
\begin{align}
    &\quad~\sum_{\tau\in\cT}\pi(\tau) |\mathbf{B}_{H:h+1}^{\theta'}(\tau_{-h})[\mathbf{B}_{h}^{\theta}(a_{h}, o_{h}) - \mathbf{B}_{h}^{\theta'}(a_{h}, o_{h})]\mathbf{v}_{h-1}^{\theta}(\tau_{h-1}) |\\
    &=\sum_{\tau_h\in (\cO\times\cA)^h: \pi(\tau_h) > 0 } \pi(\tau_h) \sum_{\tau_{-h}\in (\cO\times\cA)^{H-h}} \pi(\tau_{-h}|\tau_h) |\mathbf{B}_{H:h+1}^{\theta'}(\tau_{-h})[\mathbf{B}_{h}^{\theta}(a_{h}, o_{h}) - \mathbf{B}_{h}^{\theta'}(a_{h}, o_{h})]\mathbf{v}_{h-1}^{\theta}(\tau_{h-1}) |\\
    &\leq \sum_{\tau_h\in (\cO\times\cA)^h: \pi(\tau_h) > 0 } \pi(\tau_h) \cdot\dfrac{\sqrt{S}}{\alpha} \|[\mathbf{B}_{h}^{\theta}(a_{h}, o_{h}) - \mathbf{B}_{h}^{\theta'}(a_{h}, o_{h})]\mathbf{v}_{h-1}^{\theta}(\tau_{h-1}) \|_1
\end{align}
where for each $\pi\in \varPi$ and $\tau_h\in (\cO\times\cA)^h$, $\pi(\tau_h)$ is defined in a similar way as $\pi(\tau)$ except that $\tau_h$ is a partial trajectory. For $\pi, \tau_h$ such that $\pi(\tau_h) > 0$, $\pi(\tau_{-h}|\tau_h)$ is defined as $\pi(\tau_h, \tau_{-h}) / \pi(\tau_h)$.

We conclude that \eqref{eq:REG2} can be bounded by
\begin{align}
    &\quad~\sum_{j=1}^{K} \|\Pr_{\hat{\theta}^j}^{\Tilde{\pi}^j} - \Pr_{\theta^*}^{\Tilde{\pi}^j} \|_{\mathrm{TV}}\\
    &\leq \frac{\sqrt{S}}{2\alpha} \sum_{j=1}^{K}\left( \|\mathbf{v}_0^{\theta^*} - \mathbf{v}_0^{\hat{\theta}^j}\|_1 +\sum_{h=1}^{H-1}\sum_{\tau_h\in (\cO\times\cA)^h}\Tilde{\pi}^j(\tau_h)\left\|[\mathbf{B}_{h}^{\theta^*}(a_{h}, o_{h}) - \mathbf{B}_{h}^{\hat{\theta}^j}(a_{h}, o_{h})]\mathbf{v}_{h-1}^{\theta^*}(\tau_{h-1})\right\|_1 \right)\label{eq:step21}
\end{align}
proving the claim.

\textbf{Step 2: Bound Projected Distances of Observable Operators with \eqref{eq:MLE2}.} In this step we prove the following claim

\textbf{Claim:} With probability at least $1-\frac{1}{K}$,
\begin{align}
    (\text{OP-MLE-}0):=&\dfrac{1}{2}\|\mathbf{v}_0^{\hat{\theta}^k} - \mathbf{v}_0^{\theta^*}\|_1 \leq \sqrt{\frac{\beta}{k}} \label{eq:opmle0}\\
    (\text{OP-MLE-}h):=&\sum_{j=1}^{k} \left(\dfrac{1}{2}\sum_{\tau_h} \Tilde{\pi}^j(\tau_h) \|[\mathbf{B}_{h}^{\theta^*}(a_{h}, o_{h}) - \mathbf{B}_{h}^{\hat{\theta}^k}(a_{h}, o_{h})]\mathbf{v}_{h-1}^{\theta^*}(\tau_{h-1})\|_1\right)^2 \leq \dfrac{4S}{\alpha^2}\beta \label{eq:opmleh}
\end{align}
for all $k\in [K]$, where $\beta := 3\log(K|\Bar{\varTheta}|) + 5 + \dfrac{1}{K}$.\\~

Observe that the vectors $\mathbf{v}_0^\theta \in\mathbb{R}^{\cO} $ and $\mathbf{v}_h^{\theta}(\tau_h)\in\mathbb{R}^{\cO}$ defined in Step 1 have the following probabilistic interpretation
\begin{align}
    \mathbf{v}_0^{\theta}(o_1) &= \Pr_{\theta}(o_1)\qquad\forall o_1\in\cO\\
    \pi(\tau_h)\cdot [\mathbf{v}_h^{\theta}(\tau_h)](o_{h+1}) &= \Pr_{\theta}^\pi(\tau_h, o_{h+1})\qquad\forall \tau_h \in (\cO\times\cA)^{h}, o_{h+1} \in\cO
\end{align}

Therefore, for all $\theta, \theta'\in\varTheta, \pi\in\varPi$, we have
\begin{align}
    \frac{1}{2}\|\mathbf{v}_0^{\theta} - \mathbf{v}_0^{\theta'}\|_1 &= \frac{1}{2}\sum_{o_1}|\Pr_{\theta}(o_1) - \Pr_{\theta'}(o_1)| \leq \|\Pr_{\theta}^{\pi} - \Pr_{\theta'}^{\pi} \|_{\mathrm{TV}}\\
    \frac{1}{2}\sum_{\tau_h}\pi(\tau_h) \|\mathbf{v}_h^{\theta}(\tau_h) - \mathbf{v}_h^{\theta'}(\tau_h) \|_1 &= \frac{1}{2}\sum_{\tau_h}\sum_{o_{h+1}} |\Pr_{\theta}^\pi(\tau_h, o_{h+1}) - \Pr_{\theta'}^\pi(\tau_h, o_{h+1})|\leq \|\Pr_{\theta}^{\pi} - \Pr_{\theta'}^{\pi} \|_{\mathrm{TV}}
\end{align}
where we have used the fact that the total variation distance between two marginal distributions is no more than that between the original distributions. Therefore, we have
\begin{align}
    &\quad~k\left(\frac{1}{2}\|\mathbf{v}_0^{\hat{\theta}^k} - \mathbf{v}_0^{\theta^*}\|_1\right)^2 = \sum_{j=1}^k \left(\frac{1}{2}\|\mathbf{v}_0^{\hat{\theta}^k} - \mathbf{v}_0^{\theta^*}\|_1\right)^2 \leq \eqref{eq:MLE2} \label{eq:step22:v0bound} \\
    &\quad~\sum_{j=1}^{k}\left(\frac{1}{2}\sum_{\tau_h} \Tilde{\pi}^j(\tau_h) \|\mathbf{v}_h^{\hat{\theta}^k}(\tau_h) - \mathbf{v}_h^{\theta^*}(\tau_h) \|_1\right)^2 \leq \eqref{eq:MLE2}\label{eq:step22:vhbound}
\end{align}

We proceed to bound the projected operator distance. First, for any $\tau_h\in (\cO\times\cA)^h$ we have
\begin{align}
    &\quad~\|[\mathbf{B}_{h}^{\theta^*}(a_{h}, o_{h}) - \mathbf{B}_{h}^{\hat{\theta}^k}(a_{h}, o_{h})]\mathbf{v}_{h-1}^{\theta^*}(\tau_{h-1})\|_1\\
    &\leq \|\mathbf{B}_{h}^{\theta^*}(a_{h}, o_{h}) \mathbf{v}_{h-1}^{\theta^*}(\tau_{h-1}) - \mathbf{B}_{h}^{\hat{\theta}^k}(a_{h}, o_{h}) \mathbf{v}_{h-1}^{\hat{\theta}^k}(\tau_{h-1})\|_1 + \|\mathbf{B}_{h}^{\hat{\theta}^k}(a_{h}, o_{h})[\mathbf{v}_{h-1}^{\hat{\theta}^k}(\tau_{h-1}) - \mathbf{v}_{h-1}^{\theta^*}(\tau_{h-1})]\|_1\\
    &= \|\mathbf{v}_{h}^{\theta^*}(\tau_{h}) - \mathbf{v}_{h}^{\hat{\theta}^k}(\tau_{h})\|_1 + \|\mathbf{B}_{h}^{\hat{\theta}^k}(a_{h}, o_{h})[\mathbf{v}_{h-1}^{\hat{\theta}^k}(\tau_{h-1}) - \mathbf{v}_{h-1}^{\theta^*}(\tau_{h-1})]\|_1
\end{align}
where $\tau_{h-1}\in (\cO\times\cA)^{h-1}$ is the partial trajectory made of the first $h-1$ observations and actions in $\tau_h$.

Therefore, using the inequality $(x_1 + x_2)^2 \leq 2x_1^2 + 2x_2^2$, we have
\begin{align}
    &\quad~\sum_{j=1}^{k} \left( \frac{1}{2}\sum_{\tau_h} \Tilde{\pi}^j(\tau_h) \|[\mathbf{B}_{h}^{\theta^*}(a_{h}, o_{h}) - \mathbf{B}_{h}^{\hat{\theta}^k}(a_{h}, o_{h})]\mathbf{v}_{h-1}^{\theta^*}(\tau_{h-1})\|_1\right)^2\\
    &\leq 2\underbrace{\sum_{j=1}^{k} \left(\frac{1}{2}\sum_{\tau_h} \Tilde{\pi}^j(\tau_h) \|\mathbf{v}_{h}^{\theta^*}(\tau_{h}) - \mathbf{v}_{h}^{\hat{\theta}^k}(\tau_{h})\|_1\right)^2}_{=:\xi_1} \\
    &+2\underbrace{\sum_{j=1}^{k} \left( \frac{1}{2}\sum_{\tau_h} \Tilde{\pi}^j(\tau_h) \|\mathbf{B}_{h}^{\hat{\theta}^k}(a_{h}, o_{h})[\mathbf{v}_{h-1}^{\hat{\theta}^k}(\tau_{h-1}) - \mathbf{v}_{h-1}^{\theta^*}(\tau_{h-1})]\|_1\right)^2}_{\xi_2}
    \label{eq:decomposeprojopdist}
\end{align}

From \eqref{eq:step22:vhbound} we know that $\xi_1 \leq \eqref{eq:MLE2}$. It remains to bound $\xi_2$. We have
\begin{align}
    &\quad~\frac{1}{2}\sum_{\tau_h} \Tilde{\pi}^j(\tau_h) \|\mathbf{B}_{h}^{\hat{\theta}^k}(a_{h}, o_{h})[\mathbf{v}_{h-1}^{\hat{\theta}^k}(\tau_{h-1}) - \mathbf{v}_{h-1}^{\theta^*}(\tau_{h-1})]\|_1\\
    &= \frac{1}{2} \sum_{\tau_{h-1}} \Tilde{\pi}^j(\tau_{h-1})\sum_{a_h, o_h} \Tilde{\pi}_h^j(a_h|\tau_{h-1}, o_h)  \|\mathbf{B}_{h}^{\hat{\theta}^k}(a_{h}, o_{h})[\mathbf{v}_{h-1}^{\hat{\theta}^k}(\tau_{h-1}) - \mathbf{v}_{h-1}^{\theta^*}(\tau_{h-1})]\|_1\\
    &\stackrel{(*)}{\leq}  \dfrac{\sqrt{S}}{2\alpha} \sum_{\tau_{h-1}} \Tilde{\pi}^j(\tau_{h-1})  \|\mathbf{v}_{h-1}^{\hat{\theta}^k}(\tau_{h-1}) - \mathbf{v}_{h-1}^{\theta^*}(\tau_{h-1})\|_1
\end{align}
where in (*) we have used Lemma \ref{lem:sumBbound}.

Therefore, using \eqref{eq:step22:vhbound}, we have
\begin{align}
    \xi_2 \leq \sum_{j=1}^{k} \left(\frac{\sqrt{S}}{2\alpha}\sum_{\tau_{h-1}} \Tilde{\pi}^j(\tau_{h-1})  \|\mathbf{v}_{h-1}^{\hat{\theta}^k}(\tau_{h-1}) - \mathbf{v}_{h-1}^{\theta^*}(\tau_{h-1})\|_1\right)^2\leq \dfrac{S}{\alpha^2} \eqref{eq:MLE2}\label{eq:opmleleqmle1}
\end{align}
and consequently
\begin{align}
    &\quad~\sum_{j=1}^{k} \left(\frac{1}{2}\sum_{\tau_h} \Tilde{\pi}^j(\tau_h) \|[\mathbf{B}_{h}^{\theta^*}(a_{h}, o_{h}) - \mathbf{B}_{h}^{\hat{\theta}^k}(a_{h}, o_{h})]\mathbf{v}_{h-1}^{\theta^*}(\tau_{h-1})\|_1\right)^2 \\
    &\leq 2(\xi_1 + \xi_2) \leq 2\left(1+\dfrac{S}{\alpha^2}\right)\eqref{eq:MLE2} \leq \dfrac{4S}{\alpha^2}\eqref{eq:MLE2}.\label{eq:opmleleqmle2}
\end{align}

From Lemma \ref{lem:step1} and \eqref{eq:mleleqmle}, we know that with probability at least $1-\frac{1}{K}$, we have $\eqref{eq:MLE2} \leq 3\log(K|\Bar{\varTheta}|) + 5 + \dfrac{1}{K} =\beta$. The claim then follows from \eqref{eq:step22:v0bound} and \eqref{eq:opmleleqmle2}.

\textbf{Step 3: Bound (OP-REG) with Results on (OP-MLE) via Index Change Lemma.}
In the final step, for each $h=0,1,\cdots,H-1$, we utilize Proposition \ref{lem:exchangeindexdeluxe} to provide an upper bound of each (OP-REG-$h$) via the high-probability bound on (OP-MLE-$h$). 

Suppose that both \eqref{eq:opmle0} and \eqref{eq:opmleh} are true. First, we have
\begin{align}
    (\text{OP-REG-}0)&=\sum_{k=1}^{K} \dfrac{1}{2}\|\mathbf{v}_0^{\hat{\theta}^k} - \mathbf{v}_0^{\theta^*}\|_1\leq \sum_{k=1}^K \sqrt{\dfrac{\beta}{k}} \leq 2\sqrt{\beta K} .\label{eq:opreg0bound}
\end{align}

Now, fix $h\in[H-1]$, we will use (OP-MLE-$h$) to bound (OP-REG-$h$). 

Define the following quantities: For $k\in[K], a\in \cA, o, \Tilde{o}\in\cO, \tau_{h-1}\in (\cO\times\cA)^{h-1}$, define $ w_{k, a, o, \Tilde{o}}, x_{k, a, o, \tau_{h-1}} \in\mathbb{R}^{\cS}$ by
\begin{align}
    w_{k, a, o, \Tilde{o}}^T &:= \frac{1}{2}\mathbf{e}_{\Tilde{o}}^T [\mathbf{B}_{h}^{\theta^*}(a, o) - \mathbf{B}_{h}^{\hat{\theta}^k}(a, o)] \mathbb{Z}_h^{\theta^*}\\
    x_{k, a, o, \tau_{h-1}} &:=\Tilde{\pi}^k(\tau_h)\mathbb{Z}_h^{\theta^*\dagger} \mathbf{v}_{h-1}^{\theta^*}(\tau_{h-1})
\end{align}
where $\tau_h\in (\cO\times\cA)^h$ is a trajectory such that $a_h=a, o_h=o$, and the first $h-1$ steps are made of $\tau_{h-1}$. Then, (OP-MLE-$h$) can be written as $$\sum_{j=1}^{k} \left(\sum_{(a, o)\in \cA\times\cO} \sum_{\Tilde{o}\in\cO} \sum_{\tau_{h-1}\in (\cO\times\cA)^{h-1} } |w_{k, a, o, \Tilde{o}}^T x_{j, a, o, \tau_{h-1}}|\right)^2.$$
Note that we have used the fact that $\mathbf{v}_{h-1}^{\theta^*}(\tau_{h-1})$ is in the (column) span of $\mathbb{Z}_h^{\theta^*}$ (and hence $\mathbf{v}_{h-1}^{\theta^*}(\tau_{h-1}) = \mathbb{Z}_h^{\theta^*} \mathbb{Z}_h^{\theta^* \dagger} \mathbf{v}_{h-1}^{\theta^*}(\tau_{h-1}) $).


Furthermore, for each $k\in [K]$, we have
\begin{align}
    &\quad~\sum_{(a, o)\in\cA\times\cO}\sum_{\tau_{h-1}\in (\cO\times\cA)^{h-1}}\|x_{k, a, o, \tau_{h-1}}\|_1 \\
    &= \sum_{(a, o)\in\cA\times\cO}\sum_{\tau_{h-1}\in (\cO\times\cA)^{h-1}} \|\Tilde{\pi}^k(\tau_h) \mathbb{Z}_h^{\theta^*\dagger} \mathbf{v}_{h-1}^{\theta^*}(\tau_{h-1})\|_1\\
    &= \sum_{(a, o)\in\cA\times\cO}  \sum_{\tau_{h-1}\in (\cO\times\cA)^{h-1}} \Tilde{\pi}^k_h(a|\tau_{h-1}, o) \|\Tilde{\pi}^k(\tau_{h-1}) \mathbb{Z}_h^{\theta^*\dagger} \mathbf{v}_{h-1}^{\theta^*}(\tau_{h-1})\|_1\\
    &= \sum_{(a, o)\in\cA\times\cO}  \sum_{\tau_{h-1}\in (\cO\times\cA)^{h-1}} \Tilde{\pi}^k_h(a|\tau_{h-1}, o) \sum_{s_h\in\cS} \Pr_{\theta^*}^{\Tilde{\pi}^k}(\tau_{h-1}, s_h ) \\
    &= \sum_{o\in \cO}  \sum_{\tau_{h-1}\in (\cO\times\cA)^{h-1}} \sum_{s_h\in\cS} \Pr_{\theta^*}^{\Tilde{\pi}^k}(\tau_{h-1}, s_h )=O.
\end{align}

We also have
\begin{align}
    &\quad~\sum_{(a, o)\in\cA\times\cO}\sum_{\Tilde{o}\in\cO} \|w_{k, a, o, \Tilde{o}} \|_1\\
    &= \sum_{s\in\cS} \sum_{(a, o)\in\cA\times\cO} \left\|\frac{1}{2} [\mathbf{B}_{h}^{\theta^*}(a, o) - \mathbf{B}_{h}^{\hat{\theta}^k}(a, o)] \mathbb{Z}_h^{\theta^*} \mathbf{e}_s \right\|_{1}\\
    &\leq \sum_{s\in\cS}\dfrac{1}{2}\left(\sum_{(a, o)\in\cA\times\cO}\left\|\mathbf{B}_{h}^{\theta^*}(a, o)  \mathbb{Z}_h^{\theta^*} \mathbf{e}_s \right\|_{1} + \sum_{(a, o)\in\cA\times\cO}\left\|\mathbf{B}_{h}^{\hat{\theta}^k}(a, o)  \mathbb{Z}_h^{\theta^*} \mathbf{e}_s \right\|_{1} \right)\\
    &\stackrel{(*)}{\leq} \sum_{s\in\cS}A\dfrac{\sqrt{S}}{\alpha} \|\mathbb{Z}_h^{\theta^*} \mathbf{e}_s\|_1 = \dfrac{S^{1.5}A}{\alpha}
\end{align}
where in (*) we have used Lemma \ref{lem:sumBbound}.

Applying Proposition \ref{lem:exchangeindexdeluxe} with $\lambda = \frac{S}{\alpha^2}$, we have
\begin{align}
    (\text{OP-REG-}h)&=\dfrac{1}{2}\sum_{k=1}^{K} \sum_{\tau_h } \Tilde{\pi}^k(\tau_h) \|[\mathbf{B}_{h}^{\theta^*}(a_h, o_h) - \mathbf{B}_{h}^{\hat{\theta}^k}(a, o)]\mathbf{v}_{h-1}^{\theta^*}(\tau_{h-1})\|_1\\
    &=\sum_{k=1}^{K} \sum_{(a, o)\in \cA\times\cO} \sum_{\Tilde{o}\in\cO} \sum_{\tau_{h-1}\in (\cO\times\cA)^{h-1} } |w_{k, a, o, \Tilde{o}}^T x_{k, a, o, \tau_{h-1}}| \\
    &\leq \sqrt{\left(\dfrac{S}{\alpha^2}+\dfrac{4S}{\alpha^2} \beta\right)SAO^2K\log\left(1 + SAO^2K\right)}\\
    &\leq 3\alpha^{-1}SO\sqrt{A\beta K \log\left(1 + SAO^2K\right)}.
\end{align}

Using Step 1, we conclude that with probability at least $1-\frac{1}{K}$,
\begin{align}
    \eqref{eq:REG2}
    &\leq \dfrac{\sqrt{S}}{\alpha} \left((\text{OP-REG-0}) + \sum_{h=1}^{H-1}(\text{OP-REG-}h) \right)  \\
    &\leq \dfrac{\sqrt{S}}{\alpha} \left(2\sqrt{\beta K} +  (H-1)\cdot 3\alpha^{-1}SO\sqrt{A\beta K \log\left(1 + SAO^2K\right)}  \right)\\
    &\leq 3\alpha^{-2}H SO \sqrt{SA\beta K \log(1 + SAO^2K) }\label{eq:REGorder}.
\end{align}

Since TV distances are bounded by $1$, we also have $\eqref{eq:REG2}\leq K$ with probability 1. Therefore, we have 
\begin{equation}
    \E\left[\eqref{eq:REG2}\right] \leq 3\alpha^{-2}H SO \sqrt{SA\beta K \log(1 + SAO^2K) } + 1.
\end{equation}
Finally, using Lemma \ref{lem:step0} and \eqref{eq:regleqreg} we have
\begin{align}
    \mathrm{BReg}(\phi^{\mathrm{PS4POMDPs}}, K) &\leq 2H + H \E\left[\eqref{eq:REG}\right] \\
    &\leq  3H + H \E\left[\eqref{eq:REG2}\right] \\
    &= \bigO\left(\alpha^{-2}H^2 S O \sqrt{SA\beta K\log(1 + SAO^2K)}  \right)\\
    &= \bigOtilde\left(\alpha^{-2}H^{2} S^2 O\sqrt{HA(S A + O)K}  \right)
\end{align}
where in the last line we used the fact that $\beta = \bigOtilde(\log(|\Bar{\varTheta}|)) = \bigOtilde(HS^2 A + HSO)$.

\section{Miscellaneous Auxiliary Results}
\begin{lemma}\label{lem:logeexp}
    Let $\{Z_k\}_{k=1}^K$ be a real-valued random process adapted to the filtration $\{\mathcal{F}_k\}_{k=0}^K$. Suppose that $\psi_k(\lambda) = \log\E[\exp(\lambda Z_k)|\mathcal{F}_{k-1}]$ is well-defined for some $\lambda\in \mathbb{R}$. Then, for any $x > 0$, 
    \begin{equation}
        \Pr\left( \max_{k\in [K]}\sum_{j=1}^k [\lambda Z_j - \psi_j(\lambda) ]  < x \right) \geq 1 - e^{-x}
    \end{equation}
\end{lemma}

\begin{proof}
Define $M_0 = 1$ and $M_k = \exp\left(\sum_{j=1}^k [\lambda Z_j - \psi_j(\lambda) ] \right)$, then 
\begin{align}
    \E[M_k|\mathcal{F}_{k-1}] &= \exp\left(\sum_{j=1}^{k-1} [\lambda Z_j - \psi_j(\lambda) ] \right) \cdot \E\left[\exp(\lambda Z_k - \psi_k(\lambda)) | \mathcal{F}_{k-1} \right]\\
    &= M_{k-1}\dfrac{ \E[\exp(\lambda Z_k) | \mathcal{F}_{k-1}]}{\exp(\psi_k(\lambda))} = M_{k-1}.
\end{align}

Therefore, $M_k$ is a martingale with mean $\E[M_k] = \E[M_0] = 1$. By Doob's Martingale Inequality, we have
\begin{equation}
    \Pr\left(\max_{k\in [K]} M_k \geq e^{x}\right) \leq \E[M_K]e^{-x} = e^{-x} \qquad\forall x > 0
\end{equation}

The lemma is then established by taking the logarithm of both sides of the inequality in the condition.
\end{proof}

\begin{lemma}\label{lem:hellinger}
    Let $p, q\in\Delta_d$ be two probability distributions on the finite set $[d]$, then $\sum_{i=1}^d (\sqrt{p_i} - \sqrt{q}_i)^2 \geq \|p-q\|_{\mathrm{TV}}^2$.
\end{lemma}

\begin{proof}
    We have
    \begin{align}
        \|p-q\|_{\mathrm{TV}}^2 &= \frac{1}{4} \left(\sum_{i=1}^d |p_i - q_i|\right)^2 = \frac{1}{4}\left(\sum_{i=1}^d|(\sqrt{p_i} + \sqrt{q_i})(\sqrt{p_i} - \sqrt{q_i})|\right)^2 \\
        &\leq \dfrac{1}{4}\left(\sum_{i=1}^d (\sqrt{p_i} + \sqrt{q_i})^2  \right) \left(\sum_{i=1}^d (\sqrt{p_i} - \sqrt{q_i})^2  \right)\\
        &\leq \dfrac{1}{4}\left(\sum_{i=1}^d (2p_i + 2q_i)  \right) \left(\sum_{i=1}^d (\sqrt{p_i} - \sqrt{q_i})^2  \right) = \sum_{i=1}^d (\sqrt{p_i} - \sqrt{q_i})^2
    \end{align}
\end{proof}

\begin{lemma}\label{lem:confset}
    Let $\Bar{\varTheta}$ be as specified in Proposition \ref{lem:barvartheta} and $\ell, \Bar{\ell}^*$ be as defined in \eqref{eq:def:llh} and \eqref{eq:def:maxllh}. For each $k\in [K]$, with probability at least $1-\delta$, we have $\ell(\theta^*, \mathcal{D}_k) \geq \Bar{\ell}^*(\mathcal{D}_k) - \log(|\Bar{\varTheta}|/\delta)$.
\end{lemma}

\begin{proof}
    For any $\Bar{\theta}\in\Bar{\varTheta}$ and any $j\in [K]$, we have
    \begin{align}
         \E\left[ \dfrac{\Pr_{\Bar{\theta}}^{\pi^{j}}(\tau^{j})}{\Pr_{\theta^*}^{\pi^{j}}(\tau^{j}) }~\Big|~\mathcal{D}_{j}, \theta^*, \pi^j \right] = \sum_{\Tilde{\tau}\in\cT : \Pr_{\theta^*}^{\pi^j} (\Tilde{\tau}) > 0 } \Pr_{\theta^*}^{\pi^j} (\Tilde{\tau}) \dfrac{\Pr_{\Bar{\theta}}^{\pi^{j}}(\Tilde{\tau})}{\Pr_{\theta^*}^{\pi^{j}}(\Tilde{\tau}) } \leq 1
    \end{align}

    Therefore,
    \begin{align}
        &\quad~\E\left[\exp\left(\ell(\Bar{\theta}, \mathcal{D}_k) - \ell(\theta^*, \mathcal{D}_k) \right) \right] \\
        &= \E\left[\prod_{j=1}^{k-1} \dfrac{\Pr_{\Bar{\theta}}^{\pi^j}(\tau^j)}{\Pr_{\theta^*}^{\pi^j}(\tau^j) } ~\Big|~ \theta^*\right] = \E\left[\E\left[ \prod_{j=1}^{k-1} \dfrac{\Pr_{\Bar{\theta}}^{\pi^j}(\tau^j)}{\Pr_{\theta^*}^{\pi^j}(\tau^j) } ~\Big|~ \mathcal{D}_{k-1}, \theta^*\right] \right] \\
        &= \E\left[\prod_{j=1}^{k-2} \dfrac{\Pr_{\Bar{\theta}}^{\pi^j}(\tau^j)}{\Pr_{\theta^*}^{\pi^j}(\tau^j) } \E\left[ \dfrac{\Pr_{\Bar{\theta}}^{\pi^{k-1}}(\tau^{k-1})}{\Pr_{\theta^*}^{\pi^{k-1}}(\tau^{k-1}) }~\Big|~\mathcal{D}_{k-1}, \theta^* \right] \right] \leq \E\left[\prod_{j=1}^{k-2} \dfrac{\Pr_{\Bar{\theta}'}^{\pi^j}(\tau^j)}{\Pr_{\theta^*}^{\pi^j}(\tau^j) }  \right] \leq ... \leq 1
    \end{align}

    Using Markov Inequality, we have for any $\delta > 0$
    \begin{equation}
        \Pr\left(\ell(\Bar{\theta}, \mathcal{D}_k) - \ell(\theta^*, \mathcal{D}_k) \geq \log\left(\dfrac{1}{\delta}\right)\right) \leq \delta.
    \end{equation}

    Using Union Bound, we obtain
    \begin{equation}
        \Pr\left(\max_{\Bar{\theta} \in\Bar{\varTheta}} \ell(\Bar{\theta}, \mathcal{D}_k) - \ell(\theta^*, \mathcal{D}_k) \geq \log\left(\dfrac{1}{\delta}\right)\right) \leq |\Bar{\varTheta}| \delta.
    \end{equation}

    The lemma is then obtained by replacing $\delta$ with $\delta / |\Bar{\varTheta}|$.
\end{proof}

\begin{lemma}\label{lem:sumBbound}
Assume that Assumption \ref{assump:obsfullrank} is true and $\mathbf{B}_h^{\theta}$ is as defined in \eqref{eq:defoop}. Let $x \in [0, 1]^{\cA\times\cO}$ be such that $\sum_{a\in\cA}x(a, o) = 1$ for all $o\in\cO$. For any $\theta\in\varTheta$ and any $\mathbf{w}\in\mathbb{R}^{\cO}$, we have
\begin{equation}
    \sum_{a, o} x(a, o) \|\mathbf{B}_h^{\theta}(a, o) \mathbf{w} \|_1 \leq \dfrac{\sqrt{S}}{\alpha}\|\mathbf{w}\|_1
\end{equation}
\end{lemma}

\begin{proof}
    Fix $s\in\cS$, let $\mathbf{e}_s\in\mathbb{R}^{\cS}$ be its corresponding indicator vector. Since $\mathbb{Z}_{h+1}^\theta \mathbb{T}_{h, a}^{\theta}\mathrm{diag}(\mathbb{Z}_h^{\theta}(o, \cdot))$ is a non-negative matrix, we have
    \begin{align}
        &\quad~\|\mathbb{Z}_{h+1}^\theta \mathbb{T}_{h, a}^{\theta}\mathrm{diag}(\mathbb{Z}_h^{\theta}(o, \cdot)) \mathbf{e}_s \|_1\\
        &=\bm{1}^T \mathbb{Z}_{h+1}^\theta \mathbb{T}_{h, a}^{\theta}\mathrm{diag}(\mathbb{Z}_h^{\theta}(o, \cdot)) \mathbf{e}_s \\
        &=\bm{1}^T \mathbb{T}_{h, a}^{\theta}\mathrm{diag}(\mathbb{Z}_h^{\theta}(o, \cdot)) \mathbf{e}_s  = \bm{1}^T \mathrm{diag}(\mathbb{Z}_h^{\theta}(o, \cdot)) \mathbf{e}_s = \mathbb{Z}_h^{\theta}(o, s)\\
    \end{align}
    Therefore, for any $\mathbf{w}\in\mathbb{R}^{\cO}$
    \begin{align}
        &\quad~\sum_{a, o}x(a, o) \|\mathbf{B}_h^{\theta}(a, o) \mathbf{w} \|_1 = \sum_{a, o}x(a, o) \|\mathbb{Z}_{h+1}^\theta \mathbb{T}_{h, a}^{\theta}\mathrm{diag}(\mathbb{Z}_h^{\theta}(o, \cdot))(\mathbb{Z}_h^{\theta\dagger} \mathbf{w}) \|_1\\
        &\leq \sum_{a, o}x(a, o) \sum_{s\in\cS} \left|[\mathbb{Z}_h^{\theta\dagger} \mathbf{w}](s)\right|\cdot \|\mathbb{Z}_{h+1}^\theta \mathbb{T}_{h, a}^{\theta}\mathrm{diag}(\mathbb{Z}_h^{\theta}(o, \cdot)) \mathbf{e}_s\|_1\\
        &= \sum_{a, o}x(a, o) \sum_{s\in\cS} \left|[\mathbb{Z}_h^{\theta\dagger} \mathbf{w}](s)\right|\cdot \mathbb{Z}_h^{\theta}(o, s)= \sum_{o}\left(\sum_{a}x(a, o)\right) \sum_{s\in\cS} \left|[\mathbb{Z}_h^{\theta\dagger} \mathbf{w}](s)\right|\cdot \mathbb{Z}_h^{\theta}(o, s) \\
        &= \sum_{s\in\cS} \left|(\mathbb{Z}_h^{\theta\dagger} \mathbf{w})_s\right|\sum_{o} \mathbb{Z}_h^{\theta}(o, s)  = \sum_{s\in\cS} \left|[\mathbb{Z}_h^{\theta\dagger} \mathbf{w}](s)\right|=\|\mathbb{Z}_h^{\theta\dagger} \mathbf{w}\|_1\leq \dfrac{\sqrt{S}}{\alpha} \|\mathbf{w}\|_1
    \end{align}
    where we used Quick Fact 4 in the last inequality.
\end{proof}

\begin{lemma}\label{lem:BHhbound}
Assume that Assumption \ref{assump:obsfullrank} is true and $\mathbf{B}_{H:h+1}^{\theta}$ is as defined in \eqref{eq:defBHh}. Fix $0\leq h < H$.
For any policy $\pi\in \varPi$, any $\tau_h\in (\cO\times\cA)^{h}$ such that $\pi(\tau_h) > 0$, any $\theta\in\varTheta$, and any $\mathbf{w}\in\mathbb{R}^{\cO}$, we have
\begin{equation}
    \sum_{\tau_{-h}\in (\cO\times\cA)^{H-h} } \pi(\tau_{-h}|\tau_h) |\mathbf{B}_{H:h+1}^{\theta}(\tau_{-h}) \mathbf{w} | \leq \dfrac{\sqrt{S}}{\alpha}\|\mathbf{w}\|_1.
\end{equation}
\end{lemma}

\begin{proof}
    Fix $\tau_h\in (\cO\times\cA)^{h}$. First, when $h=H-1$, we have $\mathbf{B}_{H:h+1}^{\theta}(\tau_{-h}) = \mathbf{e}_{o_H}^T$ and
    \begin{align}
        \sum_{\tau_{-h}\in (\cO\times\cA)^{H-h} } \pi(\tau_{-h}|\tau_h) |\mathbf{B}_{H:h+1}^{\theta}(\tau_{-h}) \mathbf{w} | = \sum_{a_H, o_H} \pi_h(a_H|\tau_{H-1}, o_H) |\mathbf{w}(o_H)| = \sum_{o_H} |\mathbf{w}(o_H)| = \|\mathbf{w}\|_1.
    \end{align}

    For the rest of the proof, assume $h < H-1$. First, note that
    \begin{equation}
        \mathbf{B}_{H:h+1}^{\theta}(\tau_{-h}) = \underbrace{\mathbf{e}_{o_H}^T \mathbb{Z}_H^{\theta}\mathbb{T}_{H-1, a_{H-1}}^{\theta} \mathrm{diag}(\mathbb{Z}_{H-1}^{\theta}(o_{H-1}, \cdot)) \cdots \mathbb{T}_{h+1, a_{h+1}}^{\theta} \mathrm{diag}(\mathbb{Z}_{h+1}^{\theta}(o_{h+1}, \cdot))}_{=:\Tilde{\mathbf{B}}_{H:h+1}^{\theta}(\tau_{-h}) }\mathbb{Z}_{h+1}^{\theta\dagger}
    \end{equation}

    For any $s_{h+1}\in\cS$, we have
    \begin{align}
        \pi(\tau_{-h}|\tau_h) |\Tilde{\mathbf{B}}_{H:h+1}^{\theta}(\tau_{-h}) \mathbf{e}_{s_{h+1}} | = \pi(\tau_{-h}|\tau_h) \Pr_{\theta}^{-}(\tau_{-h}|s_{h+1})  = \Pr_{\theta}^{\pi(\tau_h, \cdot)}(\tau_{-h}|s_{h+1}) 
    \end{align}
    where $\pi(\tau_h, \cdot)$ is the partial policy for timestamps $h+1$ to $H$ obtained by fixing the historical trajectory $\tau_h$.

    Therefore, for any $\mathbf{w}\in\mathbb{R}^{\cO}$,
    \begin{align}
        &\quad~\sum_{\tau_{-h}} \pi(\tau_{-h}|\tau_h) |\mathbf{B}_{H:h+1}^{\theta}(\tau_{-h}) \mathbf{w} | =\sum_{\tau_{-h}} \pi(\tau_{-h}|\tau_h) |\Tilde{\mathbf{B}}_{H:h+1}^{\theta}(\tau_{-h}) (\mathbb{Z}_h^{\theta\dagger} \mathbf{w} ) |\\
        &\leq \sum_{\tau_{-h}} \sum_{s_{h+1}\in\cS} \pi(\tau_{-h}|\tau_h) |\Tilde{\mathbf{B}}_{H:h+1}^{\theta}(\tau_{-h}) \mathbf{e}_{s_{h+1}} | \cdot \left|(\mathbb{Z}_h^{\theta\dagger} \mathbf{w} )_{s_{h+1}}\right|=\sum_{\tau_{-h}} \sum_{s_{h+1}\in\cS} \Pr_{\theta}^{\pi(\tau_h, \cdot)}(\tau_{-h}|s_{h+1})  \left|(\mathbb{Z}_h^{\theta\dagger} \mathbf{w} )_{s_{h+1}}\right| \\
        &= \sum_{s_{h+1}\in\cS}\left(\sum_{\tau_{-h}} \Pr_{\theta}^{\pi(\tau_h, \cdot)}(\tau_{-h}|s_{h+1})\right)\cdot  \left|(\mathbb{Z}_h^{\theta\dagger} \mathbf{w} )_{s_{h+1}}\right| = \| \mathbb{Z}_h^{\theta\dagger} \mathbf{w}\|_1\leq \dfrac{\sqrt{S}}{\alpha}\|\mathbf{w}\|_1
    \end{align}
    where we used Quick Fact 4 in the last inequality.

\end{proof}

\section{Learning Multi-Agent POMDPs}\label{app:mapomdps}
In this appendix, we extend our results to learning problems on multi-agent POMDPs (MA-POMDPs), which are models that involve multiple learning agents with the same objective but different information about the system.
A multiple-agent POMDP can be characterized by a tuple $(I, \cS, \cA, \cO, H, b_1, P, Z, r)$, where $I\in \mathbb{N}$ is the number of agents; $\cS$ is the state space with $|\cS| = S$; $\cA = \prod_{i=1}^I\cA^i$ is the joint action space with $|\cA| = A$; $\cO = \prod_{i=1}^I\cO^i$ is the set of joint observations with $|\cO| = O$; $H\in\mathbb{N}$ is the horizon length; $b_1 \in\Delta(\cS)$ is the distribution of the initial state; $T = (T_h)_{h=1}^{H-1}, T_h: \cS\times\cA\mapsto \Delta(\cS)$ is the state transition kernel; $Z = (Z_h)_{h=1}^{H}, Z_h: \cS\mapsto \Delta(\cO)$ is the joint observation kernel; $r = (r_h)_{h=1}^{H}, r_h: \cO\times\cA \mapsto [0, 1]$ is the instantaneous reward function.

We assume that $I, \cS, \cA, \cO, H, r$ are known to the learning agents. The quantities $b_1, P$ and $Z$ are (in general) unknown to the agents. In the same way as in Section \ref{sec:prelim}, we assume that $b_1, P, Z$ are parameterized by a parameter $\theta\in \varTheta$. The parameter $\theta$ has a prior distribution $\nu^1\in \Delta(\varTheta)$.

In one episode, agent $i$'s individual policy is characterized by a collection of mappings $\pi_{\cdot; i} = (\pi_{h; i})_{h=1}^H, \pi_h^i: (\cO^{i}\times \cA^i)^{h-1}\times \cO^i \mapsto \cA^i$, that depends only on its individual information. A joint policy $\pi = (\pi_{\cdot; i})_{i\in [I]}$ consists of the individual policies of all agents. 
Let $\varPi$ denote the space of all deterministic joint policies. 

At the beginning of each learning episode, the learning agents share all of their action and observation history with each other. 
Each agent then picks an individual policy based on the collective history (with possible randomization over policies) and uses it during the episode. We assume that the agents have access to a common randomness source, and hence their random policy choices can be correlated.
The Bayesian regret is defined in the same way as in Section \ref{sec:prelim}, where we compare the expected total reward under the given (joint) learning algorithm against the best joint policy with respect to the true MA-POMDP parameter.

\begin{algorithm}[!ht]
   \caption{PS4MAPOMDPs: Agent $i$}
   \label{algo:mapsrl}
\begin{algorithmic}
   \STATE \textbf{Input:} Prior $\nu^1\in\Delta(\varTheta)$; Number of episodes $K$
	 \FOR{$k = 1$ to $K$} 
	    \STATE Use common randomness source to sample $\Tilde{\theta}^k \sim \nu^k$
	    \STATE Invoke the MA-POMDP solving oracle to obtain a policy $\Tilde{\pi}^k\in \arg\max_{\pi\in\varPi}(V_{\Tilde{\theta}^k}^{\pi})$ 
        \STATE Apply $\Tilde{\pi}^{k}_{\cdot; i}$ in the $k$-th episode
        \STATE {At the end of $k$-th episode, } Share the local trajectory $\tau^{k, i} = (a_{1:H; i}^{k}, o_{1:H; i}^{k})$ with other agents
        \STATE 
	    Use $\tau^k = (\tau^{k, j})_{j\in [I]}$ to compute new posterior $\nu^{k+1} \in\Delta(\varTheta)$ using \eqref{eq:nubayes}
	\ENDFOR
\end{algorithmic}
\end{algorithm}

The Posterior Sampling for Multi-Agent POMDPs (PS4MAPOMDPs) algorithm works in a similar way to its single-agent counterpart: At the beginning of each episode $k$, a common sample $\Tilde{\theta}^k$ is drawn based on the latest posterior distribution based on the \emph{collective} action and observation history, then the agents collectively invoke an MA-POMDP solving oracle to obtain a joint policy $\Tilde{\pi}^k$. Note that the above steps can be done separately by each agent {using the common randomness source}. Then each agent $i\in [I]$ uses its individual policy $\Tilde{\pi}^k_{\cdot; i}$ to take an action during the episode. 

Under the assumption of action and observation sharing at the beginning of each episode, the MA-POMDP learning problem can be seen as a single-agent problem where the learning agent knows all the actions and observations but is artificially restricted to use only policies from $\varPi$, the set of joint policies. Per Remark \ref{remark:restrict}, we obtain the following result.

\begin{proposition}\label{prop:multiagent}
    The Bayesian regret of the PS4MAPOMDPs algorithm $\phi$ applied to a general MA-POMDP learning problem satisfies
    \begin{equation}
        \mathrm{BReg}(\phi, K) \leq \bigOtilde\left(H^2\sqrt{(S^2A + SO) (OA)^H K}\right). 
    \end{equation}
    Furthermore, if Assumption 1 holds for the joint observation kernel $Z$, then 
    \begin{equation}
        \mathrm{BReg}(\phi, K) \leq \bigOtilde\left(\alpha^{-2}H^{2} S^2 O\sqrt{HA(S A + O)K}  \right).
    \end{equation}
\end{proposition}

Note that Assumption \ref{assump:obsfullrank} in multi-agent POMDP setting does not imply that individual observations are necessarily informative to a degree. It only requires the \emph{joint observation} to be informative. A special case that satisfies Assumption \ref{assump:obsfullrank} is 
the Dec-MDP model where the joint observation uncovers the underlying state perfectly.

\section{Experimental Details}
\subsection{Environment: Finite Horizon Tiger}\label{app:tiger}
We consider the following POMDP adapted from the widely used Tiger example \cite{cassandra1994acting}. In this POMDP, a contestant is faced with two doors. Behind one door is a tiger and behind the other door is a treasure. The contestant at each time faces three choices: opens either door and end the game, or listen (for the roar of the tiger) and deter the decision. However, listening is not accurate and has a cost. Unlike the classical infinite horizon version of Tiger, we consider a version with a finite horizon $H$, and the game ends instead of resets after the contestant opens a door. The experimental results are based on the version where $H = 10, \beta = 0.99$. The full description $(\cS, \cA, \cO, H, b_1^{\theta}, T^{\theta}, Z^{\theta}, r)$ is provided as follows:

\begin{itemize}
    \item \textbf{State space:} $\cS = \{\text{TL}, \text{TR}, \text{CD}, \text{CA}, \text{End} \}$ (tiger left; tiger right; contestant dead; contestant alive; game over). 

    \item \textbf{Action space:} $\cA = \{\text{listen}, \text{OL}, \text{OR}\}$ (open left; open right)

    \item \textbf{Observation space:} $\cO = \{\text{HL}, \text{HR}, \text{CD},\text{CA}, \text{End} \}$ (hear left; hear right; contestant dead; contestant alive; game over)

    \item \textbf{Parameter to be learnt:} A single parameter $\theta\in [0, 0.5]$ for the observation kernel. The initial distribution and transition kernels are known to the learning agent.

    \item \textbf{Initial state distribution:} $b_1(\text{TL}) = b_1(\text{TR}) = 0.5$

    \item \textbf{Transition:} $\mathbb{T}_h(s' | s, a ) = \mathbb{T}(s' | s, a )$ is stationary and deterministic
    \begin{itemize}
        \item End is an absorbing state, i.e. $\mathbb{T}(\text{End}|\text{End}, *) = 1$
        \item CD or CA always leads to End: $\mathbb{T}(\text{End}|\text{CD}, * ) = \mathbb{T}(\text{End}|\text{CA}, * ) = 1$ 
        \item $\mathbb{T}(\text{TL} | \text{TL}, \text{listen} ) = \mathbb{T}(\text{TR} | \text{TR}, \text{listen} ) = 1$ 
        \item $\mathbb{T}(\text{CD} | \text{TL}, \text{OL} ) = \mathbb{T}(\text{CD} | \text{TR}, \text{OR} ) = 1$
    
        $\mathbb{T}(\text{CA} | \text{TR}, \text{OL} ) = \mathbb{T}(\text{CA} | \text{TL}, \text{OR} ) = 1$
    \end{itemize}

    \item \textbf{Observation Kernel:} $\mathbb{Z}_h(o|s) = \mathbb{Z}(o|s)$ is stationary
    \begin{itemize}
        \item $\mathbb{Z}(\text{HL}| \text{TL}) = \mathbb{Z}(\text{HR}| \text{TR}) = 0.5+\theta$ for $h\in [H]$
        \item $\mathbb{Z}(\text{HR}| \text{TL}) = \mathbb{Z}(\text{HR}| \text{TR}) = 0.5-\theta$ for $h\in [H]$
        \item $\mathbb{Z}(\text{CD}|\text{CD}) = \mathbb{Z}(\text{CA}|\text{CA}) = \mathbb{Z}(\text{End} |\text{End}) = 1 $
    \end{itemize}

    \item \textbf{Reward function:} $r_h(o, a), h\in [H]$
    \begin{equation}
        r_h(o, a) = -100 \beta^{h-1} \bm{1}_{\{o=\text{CD}\}} + 10 \beta^{h-1} \cdot \bm{1}_{\{o=\text{CA}\}}  -1\cdot\beta^{h} \cdot \bm{1}_{\{a=\text{listen}\}}
    \end{equation}
    where $\beta \in (0, 1)$ is a discount factor.

\end{itemize}

Note that the action of opening door at time $H$ bears no consequence in this model.

\paragraph{Results.}
In Figure~\ref{fig:expresults_tiger}, we see the cumulative expected regret (or regret, for short) is clearly growing sub-linearly in all three \texttt{Tiger} instances. The sublinearity is further confirmed since when scaled down by $K$, the regret converges to zero. Then, when scaled by $\sqrt{K}$, the regret seems to converge to a constant thus indicating that it is scaling as $\sqrt{K}$.

\begin{figure}[!ht]
    \centering
    \includegraphics[width=\linewidth]{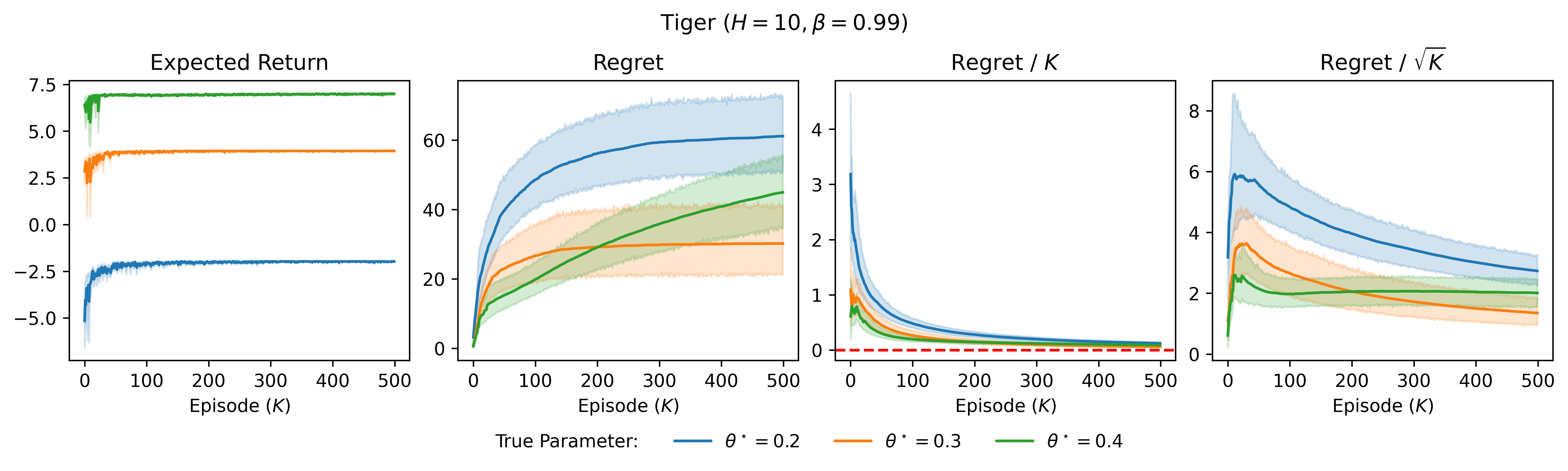}
    \caption{Performance of \texttt{PS4POMDPs} on the \texttt{Tiger} problem. The expected return for a policy is computed exactly by consider evaluating the policy on the corresponding belief-MDPs. All results are averaged over 20 independent runs (with randomly generated random seeds).}
    \label{fig:expresults_tiger}
\end{figure}

Figure \ref{fig:expresultdetailed} provides a more comprehensive look at how the algorithm learns the underlying POMDP with $\theta^* = 0.3$ in one particular run. 
We see that after some initial warm-up, the algorithm is able to quickly learn the true $\theta$ and the posterior samples $\Tilde{\theta}^k$ (shown in the lower plot in Figure \ref{fig:expresultdetailed}) become close to the true $\theta$. As a result, the policy obtained from the POMDP solver is nearly (or exactly) optimal.
\begin{figure}[!ht]
    \centering
    \includegraphics[height=4.5cm]{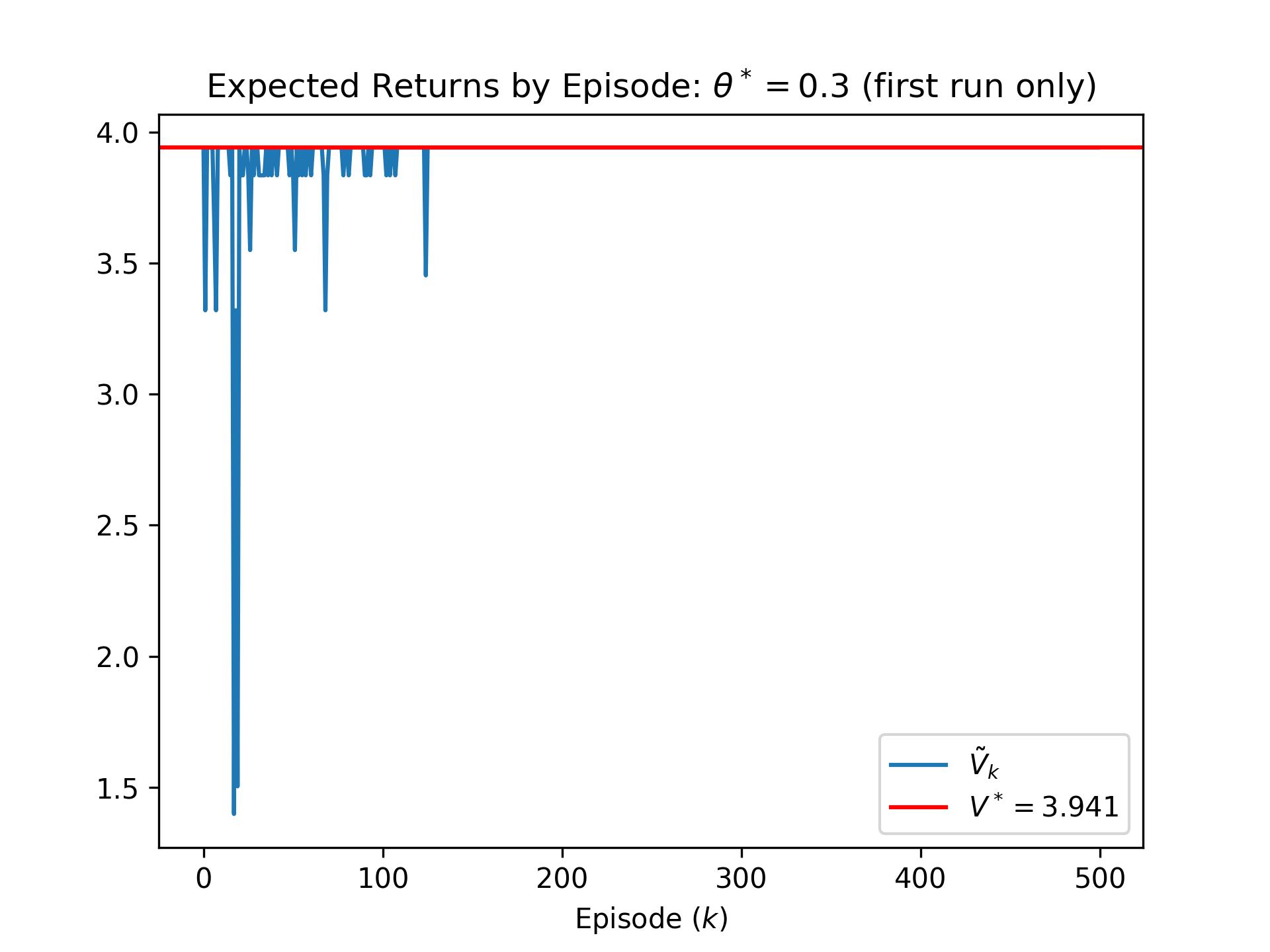}
    \includegraphics[height=4.5cm]{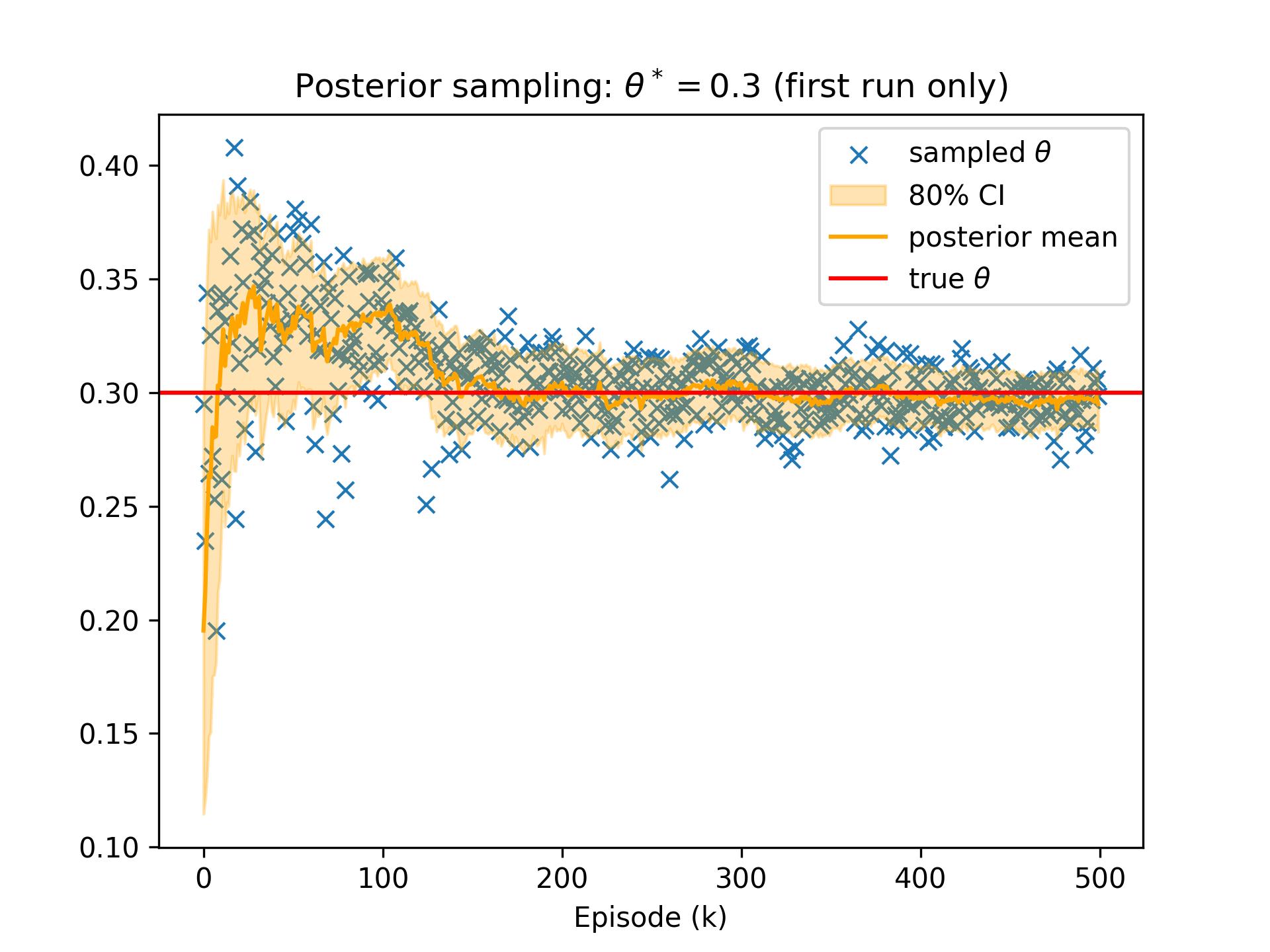}
    \caption{Detailed results for the first run of \texttt{PS4POMDPs} in \texttt{Tiger} with $\theta^* = 0.3$}
    \label{fig:expresultdetailed}
\end{figure}

\subsection{Environment: RiverSwim POMDP}\label{app:riverswim}
\paragraph{RiverSwim (POMDP) Formulation.} 
We formulate the \texttt{RiverSwim} domain \citep{strehl2008analysis, osband2013more} as a POMDP by introducing sensor errors. A \texttt{RiverSwim} POMDP instance with $L$ states is defined by the following components:
\begin{itemize}
    \item \textbf{State space:} $\mathscr{S} = \{1, 2, \dots, L\}$, where $L$ is the length of the river.
    \item \textbf{Action space:} $\mathscr{A} = \{\text{Left}, \text{Right}\}$.
    \item \textbf{Observation space:} $\mathscr{O} = \{1, 2, \dots, L\}$.
    \item \textbf{Initial state distribution:} $b_1(1) = 1$, i.e., $s_1 = 1$ almost surely.
    \item \textbf{Transition kernel:} $T_h(s'|s,a) = T(s'|s,a)$ is stationary with parameters $p=(p_1, p_2, p_3) \in \Delta(3)$ as described in the transition diagram (Figure. \ref{fig:river_swim_transitions}).
    \begin{figure}[h]
        \centering
        \includegraphics[width=0.8\textwidth]{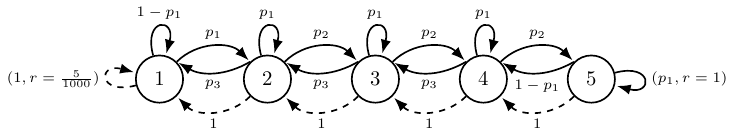}
        \caption{Transition diagram for \texttt{RiverSwim} with $L=5$ and parameterized transition kernel. Continuous and dashed arrows correspond to actions ``Right" and ``Left", respectively. (Modified from \citet{osband2013more})}
        \label{fig:river_swim_transitions}
    \end{figure}

    \item \textbf{Observation kernel:} $Z_h(o|s) = Z(o|s)$ is stationary with parameters $q=(q_1, q_2, q_3)\in \Delta(3)$:
        \begin{itemize}
            \item $Z(1|1) = 1 = Z(L|L)$
            \item $Z(3 | 2) = 1 - q_1 = Z(L-1 | L-2)$
            \item 
            For $s = 2, ..., L-2$:
            \begin{align}
                Z(o | s) = \begin{cases}
                    q_1 & \textrm{if } o=s, \\
                    q_2 & \textrm{if } o=s+1, \\
                    q_3 & \textrm{if } o=s-1.
                \end{cases}     
            \end{align}
        \end{itemize}
    In other words, the algorithm may encounter a sensor error and miscalculate it's distance to the target state $L$ when it is in the middle of the river. However, it sees its true position when it is at the starting state $1$ or when it reaches the target state $L$.

    \item \textbf{Reward function:} $r_h(o, a) = r(o,a)$ is stationary and depends only on the observation:
    \begin{equation}
        r(o, a) = \frac{5}{1000} \cdot \bm{1}_{\{o=1\}} +  1 \cdot \bm{1}_{\{o=L\}}
    \end{equation}
\end{itemize}

\paragraph{Simulated Problem Instance.}
For experiment, we use a \texttt{RiverSwim} POMDP instance with $L=6$, $p = (0.6, 0.35, 0.05)$, and $q = (0.6, 0.2, 0.2)$ and horizon $H=40$. Without the observation kernel, this environment setup is identical to \texttt{RiverSwim} MDP considered by \citet{osband2013more}. Whereas, in our POMDP environment, an algorithm will acquire an incorrect estimate of its location (i.e., state) 40\% of the time when it is in the middle of the river.

\paragraph{PS4POMDPs Variants.}
 We compare three variants of \texttt{PS4POMDPs} with different levels of prior knowledge:
 \begin{enumerate}
     \item \texttt{PS4POMDPs-KnownStruct} is given the parameterization of \texttt{RiverSwim} and thus only need to infer the parameters $p, q \in \Delta(3)$. 
     \item \texttt{PS4POMDPs-KnownObsStruct} only knows the parameterization of the observation kernel but need to infer the transition kernel from the space of all transition kernels. 
     \item \texttt{PS4POMDPs} is the generic version that uses no prior knowledge about the structure of the POMDP environment. It performs posterior inference on the space of all transition and observation kernels. 
 \end{enumerate}
 In each case, we use a diffuse Dirichlet prior with pseudo-count 1 for each parameter to be learned. 

\paragraph{Results.}
Figure \ref{fig:expresults_all_riverswim} shows the performance of \texttt{PS4POMDPs} on the \texttt{RiverSwim} problem. In all cases, the agent identifies a near-optimal policy within a few episodes. The regret again appears to be sublinear for the general \texttt{PS4POMDPs} and \texttt{PS4POMDPs-KnownObsStruct}. Although \texttt{PS4POMDPs-KnownStruct} seems to incur linear regret, the magnitude of regret is small. In fact, \texttt{PS4POMDPs-KnownStruct} shows excellent sample efficiency as it can find a near-optimal policy within merely two or three episodes. It incurs a small per-round regret (on average) due to approximation errors originated from MCMC (HMC) and the planning algorithm. Nevertheless, in Figure \ref{fig:expresults_params_riverswim}, we observe that \texttt{PS4POMDPs-KnownStruct} can identify the ground-truth parameters quickly and with high accuracy.

\begin{figure}[!ht]
    \centering
    \includegraphics[width=\linewidth]{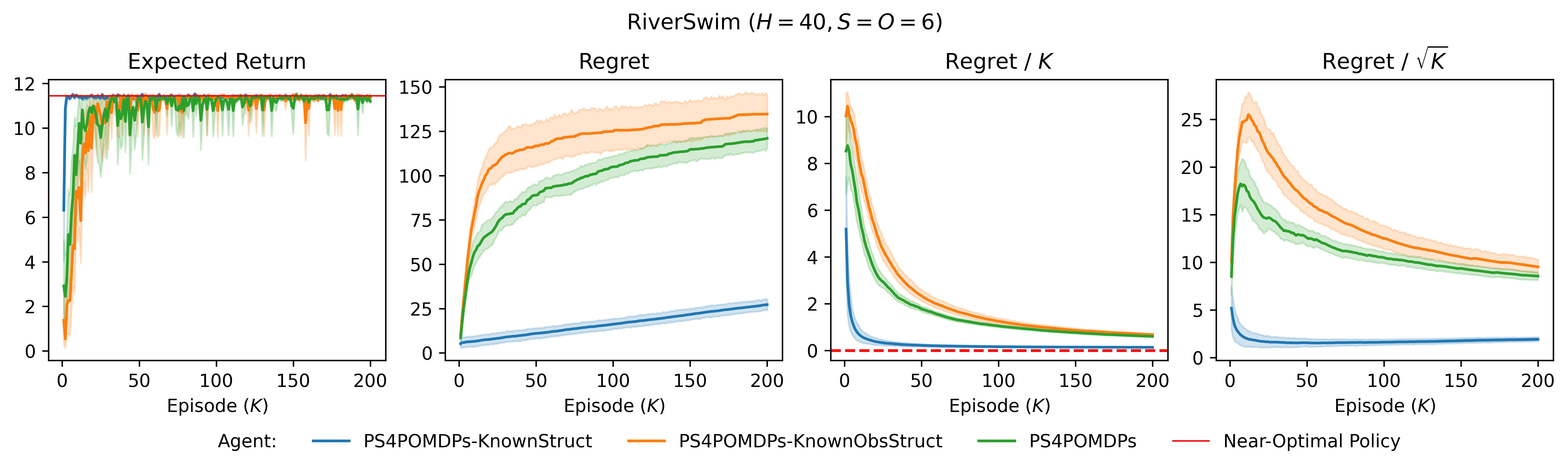}
    \caption{Performance of \texttt{PS4POMDPs} with different prior knowledge on the \texttt{RiverSwim} problem with $H=40, L=6$, and parameters $p = (0.6, 0.35, 0.05)$ and $ q= (0.6, 0.2 ,0.2)$. The expected return is computed via the Monte Carlo method (with 1000 samples) for the algorithm's policies and the near-optimal policy (red) found by SARSOP given the true model. Results are averaged over 20 independent runs and the shaded region shows 95\% confidence intervals.}
    \label{fig:expresults_all_riverswim}
\end{figure}

\begin{figure}[!ht]
    \centering
    \includegraphics[width=0.6\linewidth]{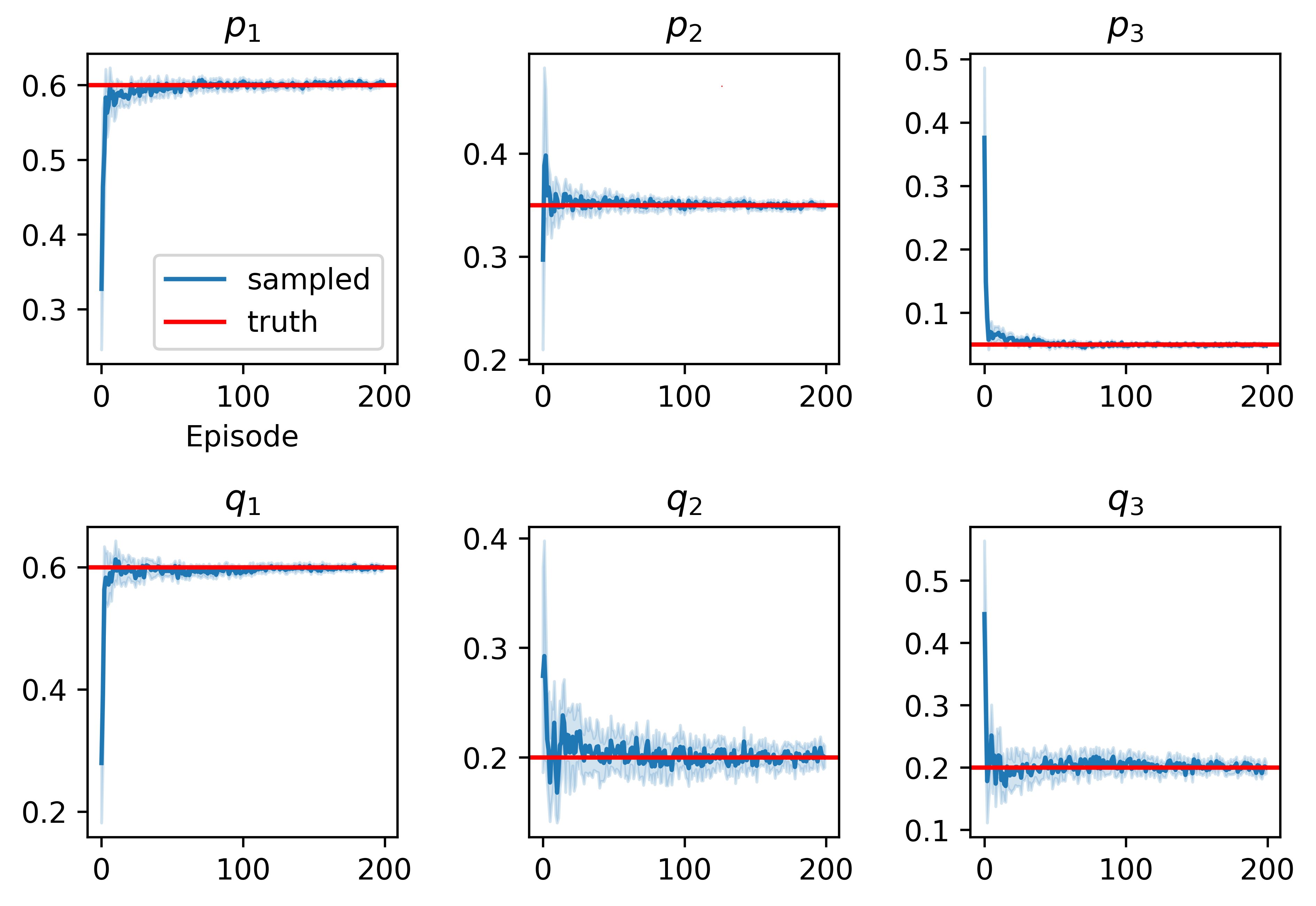}
    \caption{Sampled parameters of \texttt{PS4POMDPs-KnownStruct} in the \texttt{RiverSwim} problem with $H=40, L=6$, and ground-truth parameters $p = (0.6, 0.35, 0.05)$ and $ q= (0.6, 0.2 ,0.2)$. The shaded region shows 95\% confidence intervals across 20 independent runs.}
    \label{fig:expresults_params_riverswim}
\end{figure}

\paragraph{Additional Experiment with Random POMDPs.} We conduct further experiments with randomly generated POMDPs (\texttt{RandomPOMDP}s) derived from \texttt{RiverSwim}. Each \texttt{RandomPOMDP} has $H=40, S=10, A=10, O=10$ and is equipped with the same observation kernel and reward model as \texttt{RiverSwim}. Whereas, each \texttt{RandomPOMDP} has a stationary transition kernel randomly sampled from the (uniform) prior over all transition kernels. The results are summarized in Figure \ref{fig:expresults_all_randpomdp}. Since the  parameter space for posterior inference is considerably larger, convergence requires more episodes. However, for each random POMDP instance, we still see a clear sign of convergence in the expected return, and when scaled by $1/\sqrt{K}$, the regret appears to be converging to a constant.

\begin{figure}[!ht]
    \centering
    \includegraphics[width=\linewidth]{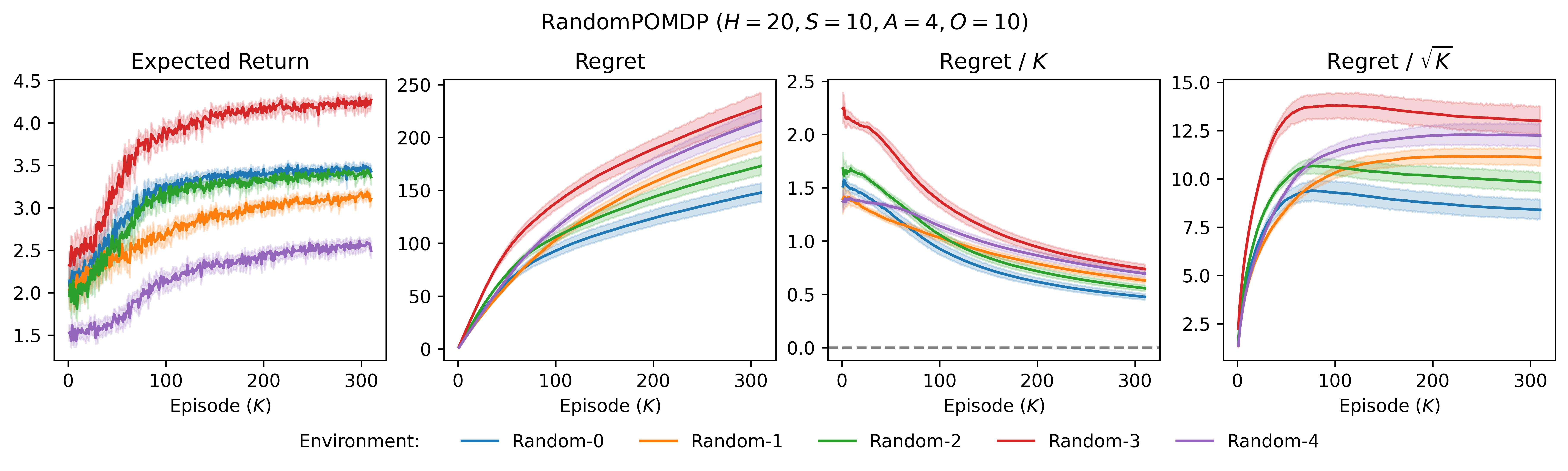}
    \caption{Performance of \texttt{PS4POMDPs} on five \texttt{RandomPOMDP}s with $H=20, S=10, A=4, O=10$. }
    \label{fig:expresults_all_randpomdp}
\end{figure}

\end{document}